\documentclass[twoside,11pt]{article}
\usepackage{arxiv_peng}
\usepackage{natbib}
\hypersetup{
    colorlinks,
    breaklinks,
    linkcolor = blue,
    citecolor = blue,
    urlcolor  = blue,
}
\usepackage{url}            
\usepackage{paralist}
\usepackage{booktabs}       
\usepackage{amsfonts}       
\usepackage{nicefrac}       
\usepackage{microtype}      
\usepackage{graphicx,subfigure,color} 
\usepackage{enumitem}
\usepackage{algorithmic,algorithm,bm}
\usepackage{lipsum}
\usepackage{amsmath,mathrsfs,amssymb,hyperref}
\usepackage{caption2}
\usepackage{dsfont}
\usepackage{multirow}
\usepackage{booktabs}       
\usepackage{standalone}     
\usepackage{preview}
\usepackage{appendix}

\usepackage{xcolor}
\allowdisplaybreaks

\usepackage{amsthm}
\allowdisplaybreaks

\usepackage{algorithm,algorithmic}

\usepackage{algorithm,algorithmic}

\DeclareMathOperator*{\argmin}{arg\,min}
\renewcommand{\tilde}{\widetilde}
\renewcommand{\hat}{\widehat}

\def \E {\mathbb{E}}
\def \x {\mathbf{x}}
\def \y {\mathbf{y}}

\def \u {\mathbf{u}}
\def \H {\mathcal{H}}

\def \R {\mathbb{R}}

\def \F {\mathcal{F}}

\def \v {\mathbf{v}}

\def \gt {\widetilde{g}}

\def \fh {\widehat{f}}

\def \s {\mathbf{s}}

\def \T {\mathrm{T}}
\def \F {\mathcal{F}}

\def \fh {\widehat{f}}

\def \X {\mathcal{X}}

\usepackage{appendix}
\newtheorem{myThm}{Theorem}

\newtheorem{myLemma}{Lemma}

\newtheorem{myProp}{Proposition}
\newtheorem{myProperty}{Property}

\theoremstyle{definition}
\newtheorem{myAssum}{Assumption}

\newtheorem{myRemark}{Remark}

\usepackage{mathtools}
\let\norm\undefined 
\DeclarePairedDelimiter\norm{\lVert}{\rVert}
\DeclarePairedDelimiter\abs{\lvert}{\rvert}
\newcommand\inner[2]{\langle #1, #2 \rangle}
\def \expert {\mathtt{expert}\mbox{-}\mathtt{regret}}
\def \meta {\mathtt{meta}\mbox{-}\mathtt{regret}}
\def \ind {\mathds{1}}

\DeclareMathOperator*{\AReg}{A-Regret}

\definecolor{DSgray}{cmyk}{0,1,0,0}

\firstpageno{1}

\begin{document}

\title{Bandit Convex Optimization in Non-stationary Environments}

\author{\name Peng Zhao \email zhaop@lamda.nju.edu.cn \\
       \name Guanghui Wang \email wanggh@lamda.nju.edu.cn \\
       \name Lijun Zhang \email zhanglj@lamda.nju.edu.cn \\
       \name Zhi-Hua Zhou \email zhouzh@lamda.nju.edu.cn \\
       \addr National Key Laboratory for Novel Software Technology\\
       Nanjing University, Nanjing 210023, China}
\maketitle

\begin{abstract}
Bandit Convex Optimization (BCO) is a fundamental framework for modeling sequential decision-making with partial information, where the only feedback available to the player is the one-point or two-point function values. In this paper, we investigate BCO in non-stationary environments and choose the \emph{dynamic regret} as the performance measure, which is defined as the difference between the cumulative loss incurred by the algorithm and that of any feasible comparator sequence. Let $T$ be the time horizon and $P_T$ be the path-length of the comparator sequence that reflects the non-stationarity of environments. We propose a novel algorithm that achieves $O(T^{3/4}(1+P_T)^{1/2})$ and $O(T^{1/2}(1+P_T)^{1/2})$ dynamic regret respectively for the one-point and two-point feedback models. The latter result is optimal, matching the $\Omega(T^{1/2}(1+P_T)^{1/2})$ lower bound established in this paper. Notably, our algorithm is more adaptive to non-stationary environments since it does not require prior knowledge of the path-length $P_T$ ahead of time, which is generally unknown.
\end{abstract}

\section{Introduction}
\label{sec:introduction}
Online Convex Optimization (OCO) is a powerful tool for modeling sequential decision-making problems, which can be regarded as an iterative game between the player and environments~\citep{book'12:Shai-OCO}. At iteration $t$, the player commits a decision $\x_t$ from a convex feasible set $\X \subseteq \R^d$, simultaneously, a convex function $f_t: \X \mapsto \R$ is revealed by environments, and then the player will suffer an instantaneous loss $f_t(\x_t)$. The standard performance measure is the \emph{regret},
\begin{equation}
  \label{eq:static-regret}
  \mbox{S-Regret}_T = \sum_{t=1}^{T} f_t(\x_t) - \min_{\x \in \X} \sum_{t=1}^{T} f_t(\x)
\end{equation}
which is the difference between the cumulative loss of the player and that of the best \emph{fixed} decision in hindsight. To emphasize the fact that the comparator in~\eqref{eq:static-regret} is fixed, it is called \emph{static} regret.

There are two setups for online convex optimization according to the information that environments reveal~\citep{book'16:Hazan-OCO}. In the \emph{full-information} setup, the player has all the information of the function $f_t$, including the gradients of $f_t$ over $\X$. By contrast, in the \emph{bandit} setup, the instantaneous loss is the only feedback available to the player. In this paper, we focus on the latter case, which is referred to as the bandit convex optimization (BCO).

BCO has attracted considerable attention because it successfully models many real-world scenarios where the feedback available to the decision maker is partial or incomplete~\citep{book'16:Hazan-OCO}. The key challenge lies in the limited feedback, i.e., the player has no access to gradients of the function. In the standard \emph{one-point feedback} model, the only feedback is the one-point function value, based on which~\citet{conf/soda/FlaxmanKM05} constructed an unbiased estimator of the gradient and then appealed to the online gradient descent algorithm that developed in the full-information setting~\citep{ICML'03:zinkvich} to establish an $O(T^{3/4})$ expected regret. Another common variant is the \emph{two-point feedback} model, where the player is allowed to query function values of two points at each iteration. \citet{conf/colt/AgarwalDX10} demonstrated an optimal $O(\sqrt{T})$ regret for convex functions under this feedback model. Algorithms and regret bounds are further developed in later studies~\citep{AISTATS'11:smooth-BCO,conf/nips/HazanL14,conf/colt/bubeck15,conf/nips/DekelEK15,conf/nips/YangM16,conf/stoc/BubeckLE17}.

\begin{table}[!t]
\centering
\small
\caption{Comparisons of dynamic regret for BCO problems. In the table, the column of ``Parm-Free'' indicates whether the algorithm requires to know the path-length in advance. Meanwhile, $T$ is the time horizon, $P_T = P_T(\u_1,\ldots,\u_T)$ and $P^*_T = \max_{\x_1,\ldots,\x_T \in \X} P_T(\x_1,\ldots,\x_T)$.}\vspace{3mm}
\label{table:dynamic-regret-BCO}
\resizebox{0.97\textwidth}{!}{
\begin{tabular}{ccccc}\toprule
\textbf{Feedback model}   & \textbf{Dynamic regret} & \textbf{Type} & \textbf{Parm-Free} & \textbf{Reference}  \\ \midrule
one-point   &  $O\big(T^{\frac{3}{4}}(1+P_T^*) \big)$   & worst-case  & NO & \citep{journal'19:BCO-IoT}\\
one-point   &  $O\big(T^{\frac{3}{4}}(1+P_T)^{\frac{1}{2}} \big)$   & universal & YES & This work \\ \midrule
two-point   & $O\big(\sqrt{T (1 + P_T^*)}\big)$ & worst-case & NO & \citep{ICML'16:Yang-smooth}\\
two-point   &  $O\big(\sqrt{T}(1 + P_T^*) \big)$  & worst-case & NO & \citep{journal'19:BCO-IoT} \\
two-point     &  $O\big(\sqrt{T(1+P_T)}\big)$ & universal & YES & This work \\\bottomrule
\end{tabular}}
\end{table}

Note that the static regret in~\eqref{eq:static-regret} compares with a fixed benchmark, so it implicitly assumes that there is a reasonably good decision over all iterations. Unfortunately, this may not be true in non-stationary environments, where the underlying distribution of online functions changes. To address this limitation, the notion of \emph{dynamic regret} is introduced by~\citet{ICML'03:zinkvich} and defined as the difference between the cumulative loss of the player and that of a comparator sequence $\u_1,\ldots,\u_T \in \X$,
\begin{equation}
  \label{eq:dynamic-regret}
  \mbox{D-Regret}_T(\u_1,\ldots,\u_T) = \sum_{t=1}^{T} f_t(\x_t) - \sum_{t=1}^{T} f_t(\u_t).
\end{equation}
In contrast to a fixed benchmark in the static regret, dynamic regret compares with a \emph{changing} comparator sequence and therefore is more suitable in non-stationary environments. We remark that~\eqref{eq:dynamic-regret} is also called the \emph{universal} dynamic regret, since it holds universally for any feasible comparator sequence. In the literature, there is a variant named the \emph{worst-case} dynamic regret~\citep{OR'15:dynamic-function-VT}, which specifies the comparator sequence to be minimizers of online functions, namely, $\u_t = \x_t^* \in \argmin_{\x\in \X} f_t(\x)$. As pointed out by~\citet{NIPS'18:Zhang-Ader}, the universal dynamic regret is more desired, because the worst-case dynamic regret is typically too pessimistic while the universal one is more adaptive to the non-stationarity of environments. Moreover, the universal dynamic regret is more general since it accommodates the worst-case dynamic regret and static regret as special cases.

Recently, there are some studies on the dynamic regret of BCO problems~\citep{ICML'16:Yang-smooth,journal'19:BCO-IoT}. They provide the worst-case dynamic regret only, and the algorithms require some quantities as the input which are generally unknown in advance. Therefore, it is desired to design algorithms that enjoy \emph{universal} dynamic regret for BCO problems.

In this paper, we start with the bandit gradient descent (BGD) algorithm of~\citet{conf/soda/FlaxmanKM05}, and analyze its universal dynamic regret. We demonstrate that the optimal parameter configuration of vanilla BGD also requires prior information of the unknown path-length. To address this issue, we propose the Parameter-free Bandit Gradient Descent algorithm (\textsc{PBGD}), which is inspired by the strategy of maintaining multiple learning rates~\citep{NIPS'16:MetaGrad}. Our approach is essentially an online ensemble method~\citep{book'12:ensemble-zhou}, consisting of meta-algorithm and expert-algorithm. The basic idea is to maintain a pool of candidate parameters, and then invoke multiple instances of the expert-algorithm simultaneously, where each expert-algorithm is associated with a candidate parameter. Next, the meta-algorithm combines predictions from expert-algorithms by an expert-tracking algorithm~\citep{book/Cambridge/cesa2006prediction}. However, it is prohibited to run multiple expert-algorithms with different parameters simultaneously in BCO problems, since the player is only allowed to query one/two points in the bandit setup. To overcome this difficulty, we carefully design a surrogate function, as the linearization of the smoothed version of the loss function in the sense of expectation, and make the strategy suitable for bandit convex optimization. Our algorithm and analysis accommodate one-point and two-point feedback models, and Table~\ref{table:dynamic-regret-BCO} summarizes existing dynamic regret for BCO problems and our results. The main contributions of this work are listed as follows.
\begin{compactitem}
  \item We establish the first \emph{universal} dynamic regret that supports to compare with any feasible comparator sequence for the bandit gradient descent algorithm, in a unified analysis framework. 
  \item We propose a \emph{parameter-free} algorithm, which does not require to know the upper bound of the path-length $P_T$ ahead of time, and meanwhile enjoys the state-of-the-art dynamic regret.
  \item We establish the \emph{first} minimax lower bound 
  of universal dynamic regret for BCO problems.
\end{compactitem}

The rest of the paper is structured as follows. Section~\ref{sec:related-work}  briefly reviews related work. In Section~\ref{sec:BGD}, we introduce the bandit gradient descent algorithm for BCO problems and provide the dynamic regret analysis. Section~\ref{sec:parameter-free-BGD} presents the parameter-free BGD algorithm, the main contribution of this paper, with dynamic regret analysis. Next, in Section~\ref{sec:lower-bound-extension}, we establish the lower bound and provide several extensions. Section~\ref{sec:analysis-BGD} and Section~\ref{sec:analysis-PBGD} present the proofs of main results. Section~\ref{sec:conclusion} concludes the paper and discusses future directions.

\section{Related Work}
\label{sec:related-work}
We briefly introduce related work of bandit convex optimization and dynamic regret.

\subsection{Bandit Convex Optimization}

In the bandit convex optimization setting, the player is only allowed to query function values of one point or two points, and the gradient information is not accessible as opposed to the full-information setting.

For the one-point feedback model, the seminal work of~\citet{conf/soda/FlaxmanKM05} constructed an unbiased gradient estimator and established an $O(T^{3/4})$ expected regret for convex and Lipschitz functions. A similar result was independently obtained by~\citet{conf/nips/Kleinberg04}. Later, an $O(T^{2/3})$ rate was shown to be attainable with either strong convexity~\citep{conf/colt/AgarwalDX10} or smoothness~\citep{AISTATS'11:smooth-BCO}. When functions are both strongly convex and smooth,~\citet{conf/nips/HazanL14} designed a novel algorithm that achieves a regret of $O(\sqrt{T\log T})$ based on the follow-the-regularized-leader framework with self-concordant barriers, matching the $\Omega(\sqrt{T})$ lower bound~\citep{conf/colt/Shamir13} up to logarithmic factors. Furthermore, recent breakthroughs~\citep{conf/colt/bubeck15,conf/stoc/BubeckLE17} showed that $O(\mbox{ploy}(\log T)\sqrt{T})$ regret is attainable for convex and Lipschitz functions, though with a high dependence on the dimension $d$.

BCO with two-point feedback is proposed and studied by~\citet{conf/colt/AgarwalDX10}, and is also independently studied in the context of stochastic optimization~\citep{Nesterov'2011}.~\citet{conf/colt/AgarwalDX10} first establish the expected regret of $O(d^2\sqrt{T})$ and $O(d^2\log T)$ for convex Lipschitz and strongly convex Lipschitz functions, respectively. These bounds are proved to be minimax optimal in $T$~\citep{conf/colt/AgarwalDX10}, and the dependence on $d$ is later improved to be optimal~\citep{journals/jmlr/Shamir17}.

Besides, bandit linear optimization is a special case of BCO where the feedback is assumed to be a linear function of the chosen decision, and has been studied extensively~\citep{awerbuch2004adaptive,mcmahan2004online,NIPS'07:bandit-lower-bound,Competing:Dark,Minimax:Linear}.

\subsection{Dynamic Regret}
There are two types of dynamic regret as aforementioned. The universal dynamic regret holds universally for any feasible comparator sequence, while the worst-case one only compares with the sequence of the minimizers of online functions.

For the universal dynamic regret, existing results are only limited to the full-information setting.~\citet{ICML'03:zinkvich} showed that OGD achieves an $O(\sqrt{T}(1+P_T))$ regret, where $P_T = P_T(\u_1,\ldots,\u_T)$ is the path-length of comparator sequence $\u_1,\ldots,\u_T$,
\begin{equation}
  \label{eq:path-length}
  P_T(\u_1,\ldots,\u_T) = \sum_{t=2}^{T} \norm{\u_{t-1} - \u_t}_2.
\end{equation}
Recently,~\citet{NIPS'18:Zhang-Ader} demonstrated that this upper bound is not optimal by establishing an $\Omega(\sqrt{T(1+P_T)})$ lower bound, and further proposed an algorithm that attains an optimal $O(\sqrt{T(1+P_T)})$ dynamic regret for convex functions. However, there is no universal dynamic regret in the bandit setting.

For the worst-case dynamic regret, there are many studies in the full-information setting~\citep{OR'15:dynamic-function-VT,AISTATS'15:dynamic-optimistic,ICML'16:Yang-smooth,CDC'16:dynamic-sc,NIPS'17:zhang-dynamic-sc-smooth} as well as a few works in the bandit setting~\citep{NIPS'14:non-stationary-MAB,ICML'16:Yang-smooth,COLT'18:non-stationary-contextual,COLT'19:dynamic-MAB,AISTATS'19:window-LB,journal'19:BCO-IoT,AISTATS'20:restart}. In the bandit convex optimization, when the upper bound of $P_T^*$ is known,~\citet{ICML'16:Yang-smooth} established an $O(\sqrt{T(1+P_T^*)})$ dynamic regret for the two-point feedback model. Here, $P^*_T = \max_{\x_1,\ldots,\x_T \in \X} P_T(\x_1,\cdots,\x_T)$ is the longest path-length of the feasible comparator sequence. Later,~\citet{journal'19:BCO-IoT} applied BCO techniques in the dynamic Internet-of-Things management, showing $O(T^{3/4}(1+P_T^*))$  and $O(T^{1/2}(1+P_T^*))$ dynamic regret bounds respectively for one-point and two-point feedback models.

Another closely related performance measure for online convex optimization in non-stationary environments is the \emph{adaptive regret}~\citep{ICML'09:Hazan-adaptive}, which is defined as the maximum of ``local'' static regret in \emph{every} time interval $[q,s] \subseteq [T]$,
\begin{equation*}
{\mbox{A-Regret}}_T=\max\limits_{[q,s]\subseteq [T]} \sum_{t=q}^{s}f_t(\x_t)-\min\limits_{\x\in\mathcal{X}}\sum_{t=q}^{s}f_t(\x).
\end{equation*}
\citet{ICML'09:Hazan-adaptive} proposed an efficient algorithm that enjoys $O(\sqrt{T\log^3 T})$ and $O(d\log^2 T)$ regrets for convex and exponentially concave functions, respectively. The rate for convex functions was improved later~\citep{ICML'15:Daniely-adaptive,AISTATS'17:coin-betting-adaptive}. Moreover, \citet{ICML'18:zhang-dynamic-adaptive} investigated the relation between adaptive regret and the worst-case dynamic regret.

\section{Bandit Gradient Descent (BGD)}
\label{sec:BGD}
In this section, we provide assumptions used in the paper, then present the bandit gradient descent (BGD) algorithm for BCO problems, as well as its universal dynamic regret. To the best of our knowledge, this is the first work that analyzes the universal dynamic regret of BGD. 

\subsection{Assumptions}
\label{sec:assump}
We make following common assumptions for bandit convex optimization~\citep{conf/soda/FlaxmanKM05,conf/colt/AgarwalDX10}.
\begin{myAssum}[Bounded Region]
\label{assum:bounded-region}
The feasible set $\X$ contains the ball of radius $r$ centered at the origin and is contained in the ball of radius $R$, namely,
\begin{equation}
  \label{eq:bounded-region}
  r\mathbb{B} \subseteq \X \subseteq R\mathbb{B}
\end{equation}
where $\mathbb{B} = \{\x \in \R^d \mid \norm{\x}_2 \leq 1\}$.
\end{myAssum}

\begin{myAssum}[Bounded Function Value]
\label{assum:bounded-func-value}
The absolute values of all the functions are bounded by $C$, namely,
\begin{equation}
  \label{eq:func-value}
  \forall t\in [T], \quad \max_{\x\in \mathcal{X}} \abs{f_t(\x)} \leq C.
\end{equation}
\end{myAssum}


\begin{myAssum}[Lipschitz Continuity]
\label{assum:lipschitz-continuity}
All the functions are $L$-Lipschitz continuous over domain $\X$, that is, for all $\x, \y \in \X$, we have
\begin{equation}
  \label{eq:lipschitz}
  \forall t\in [T], \quad \lvert f_t(\x) - f_t(\y)\rvert \leq  L \norm{\x - \y}_2.
\end{equation}
\end{myAssum}


Meanwhile, we consider loss functions and the comparator sequence are chosen by an oblivious adversary.

\subsection{Algorithm and Regret Analysis}
In this part, we present algorithm and regret analysis of the bandit gradient descent.

We start from the online gradient descent (OGD) developed in the full-information setting~\citep{ICML'03:zinkvich}. OGD begins with any $\x_1 \in \X$ and performs
\begin{equation}
  \label{eq:OGD-update}
  \x_{t+1} = \mbox{Proj}_{\X} [\x_t - \eta \nabla f_t(\x_t)]
\end{equation}
where $\eta>0$ is the step size and $\mbox{Proj}_{\X}[\cdot]$ denotes the projection onto the nearest point in $\X$.

The key challenge of BCO problems is the lack of gradients. Therefore,~\citet{conf/soda/FlaxmanKM05} and~\citet{conf/colt/AgarwalDX10} propose to replace $\nabla f_t(\x_t)$ in~\eqref{eq:OGD-update} with a gradient estimator $\tilde{g}_t$, obtained by evaluating the function at one (in the one-point feedback model) or two random points (in the two-point feedback model) around $\x_t$. Details will be presented later. We unify their algorithms in Algorithm~\ref{alg:BCO}, called the Bandit Gradient Descent (\textsc{BGD}). Notice that in lines 8 and 14 of the algorithm, the projection of $\y_{t+1}$ is on a slightly smaller set $(1-\alpha)\X$ instead of $\X$, to ensure that the final decision $\x_{t+1}$ lies in the feasible set $\X$. In the following, we describe the gradient estimator and analyze the universal dynamic regret for each model.

\paragraph{\underline{One-Point Feedback Model}.}~\citet{conf/soda/FlaxmanKM05} propose the following gradient estimator,
\begin{equation}
  \label{eq:gradient-estimator-one-point}
  \tilde{g}_t = \frac{d}{\delta} f_t(\y_t + \delta\s_t)\cdot \s_t
\end{equation}
where $\s_t$ is a unit vector selected uniformly at random and $\delta > 0$ is the perturbation parameter. Then, the following lemma~\citep[Lemma 2.1]{conf/soda/FlaxmanKM05} guarantees that~\eqref{eq:gradient-estimator-one-point} is an unbiased gradient estimator of the smoothed version of the loss function $f_t$.

\begin{myLemma}
\label{lemma:gradient-est}
For any convex (but not necessarily differentiable) function $f : \mathcal{X} \mapsto \mathbb{R}$, define its \emph{smoothed} version $\hat{f}(\x) = \E_{\mathbf{v}\in \mathbb{B}} [f(\x + \delta \mathbf{v})]$. Then, for any $\delta > 0$,
\begin{equation}
  \label{eq:gradient-est}
  \E_{\mathbf{s}\in \mathbb{S}}[f(\x + \delta \mathbf{s})\cdot \mathbf{s}] = \frac{\delta}{d} \nabla \hat{f}(\x)
\end{equation}
where $\mathbb{S}$ is the unit sphere centered around the origin, namely, $\mathbb{S} = \{\x \in \R^d | \norm{\x}_2 = 1\}$.
\end{myLemma}
Therefore, we adopt $\tilde{g}_t$ to perform the online gradient descent in~\eqref{eq:OGD-update}. The main update procedures of the one-point feedback model are summarized in the case 1 (line 4-7) of Algorithm~\ref{alg:BCO}. We have the following result regarding its universal dynamic regret.
\begin{myThm}
\label{thm:dynamic-regret-one-point}
Under Assumptions~\ref{assum:bounded-region},~\ref{assum:bounded-func-value}, and~\ref{assum:lipschitz-continuity}, for any $\delta>0$, $\eta >0$, and $\alpha = \delta/r$, the expected dynamic regret of \textsc{BGD}$(T, \delta,\alpha,\eta)$ for the one-point feedback model satisfies
\begin{equation}
  \label{eq:dynamic-regret-one-point}
  \E\left[\sum_{t=1}^T f_t(\x_t)\right] - \sum_{t=1}^T f_t(\u_t) \leq \frac{7R^2 + RP_T}{4\eta} + \frac{ \eta d^2 C^2T}{2\delta^2} + \big(3L + \frac{LR}{r}\big)\delta T,
\end{equation}
for \emph{any} feasible comparator sequence $\u_1,\ldots,\u_T \in \X$.
\end{myThm}

\begin{myRemark}
\label{remark:one-point}
By setting $\eta = ((7R^2 + RP_T)/T)^{3/4}$ and $\delta = \eta^{1/3}$, we obtain an $O(T^{3/4}(1 + P_T)^{1/4})$ dynamic regret. However, such a configuration requires prior knowledge of $P_T$, which is generally unavailable. We will develop a parameter-free algorithm to eliminate the undesired dependence later.
\end{myRemark}

\begin{algorithm}[!t]
   \caption{Bandit Gradient Descent (\textsc{BGD})}
   \label{alg:BCO}
\begin{algorithmic}[1]
   \REQUIRE time horizon $T$, perturbation parameter $\delta$, shrinkage parameter $\alpha$, step size $\eta$
   \STATE {Let $\y_1 = \mathbf{0}$}
   \FOR{$t=1$ {\bfseries to} $T$}
      \STATE {Select a unit vector $\mathbf{s}_t$ uniformly at random\\
      \{\textbf{Case 1.} One-Point Feedback Model\}}
        \STATE Submit $\x_{t} = \y_{t} + \delta \s_t $
        \STATE Receive $f_t(\x_t)$ as the feedback
        \STATE Construct the gradient estimator by~\eqref{eq:gradient-estimator-one-point}
        \STATE {$\y_{t+1} = \mbox{Proj}_{(1-\alpha) \X} [\y_t - \eta \tilde{g}_t]$\\
        \{\textbf{Case 2.} Two-Point Feedback Model\}}
        \STATE Submit $\x^{(1)}_{t} = \y_{t} + \delta \mathbf{s}_t$ and $\x^{(2)}_{t} = \y_{t} - \delta \mathbf{s}_t$
        \STATE Receive $f_{t}(\x^{(1)}_{t})$ and $f_{t}(\x^{(2)}_{t})$ as the feedback
        \STATE Construct the gradient estimator by~\eqref{eq:gradient-estimator-two-point}
        \STATE {$\y_{t+1} = \mbox{Proj}_{(1-\alpha) \mathcal{X}} [\y_t - \eta \tilde{g}_t]$}
    \ENDFOR
\end{algorithmic}
\end{algorithm}

\paragraph{\underline{Two-Point Feedback Model}.} In this setup, the player is allowed to query two points, $\x^{(1)}_t = \y_t + \delta \s_t$ and $\x^{(2)}_t = \y_t - \delta \s_t$. Then, the function values $f_t(\x^{(1)}_t)$ and $f_t(\x^{(2)}_t)$ are revealed as the feedback. We use the following gradient estimator~\citep{conf/colt/AgarwalDX10},
\begin{equation}
  \label{eq:gradient-estimator-two-point}
  \tilde{g}_t = \frac{d}{2\delta} \left(f_t(\y_t + \delta\s_t) - f_t(\y_t - \delta\s_t)\right)\cdot \s_t.
\end{equation}
The major limitation of the one-point gradient estimator~\eqref{eq:gradient-estimator-one-point} is that it has a potentially large magnitude, proportional to the $1/\delta$ which is usually quite large since the perturbation parameter $\delta$ is typically small. This is avoided in the two-point gradient estimator~\eqref{eq:gradient-estimator-two-point}, whose magnitude can be upper bounded by $Ld$, independent of the perturbation parameter $\delta$. This crucial advantage leads to the substantial improvement in the dynamic regret (also static regret).

\begin{myThm}
\label{thm:dynamic-regret-two-point}
Under Assumptions~\ref{assum:bounded-region},~\ref{assum:bounded-func-value}, and~\ref{assum:lipschitz-continuity}, for any $\delta > 0$, $\eta>0$, and $\alpha = \delta/r$, the expected dynamic regret of \textsc{BGD}$(T, \delta,\alpha,\eta)$ for the two-point feedback model satisfies
\begin{equation}
  \label{eq:dynamic-regret-two-point}
  \E\left[\sum_{t=1}^T \frac{1}{2}\big(f_t(\x^{(1)}_t) + f_t(\x^{(2)}_t)\big)\right] - \sum_{t=1}^T f_t(\u_t) \leq \frac{7R^2 + RP_T}{4\eta} + \frac{\eta L^2 d^2}{2}T + \big(3L + \frac{LR}{r}\big)\delta T
\end{equation}
for \emph{any} feasible comparator sequence $\u_1,\ldots,\u_T \in \mathcal{X}$.
\end{myThm}

\begin{myRemark}
By setting $\eta = \sqrt{(7R^2 + RP_T)/(2L^2d^2T)}$ and $\delta = 1/\sqrt{T}$, BGD algorithm achieves an $O(T^{1/2}(1+P_T)^{1/2})$ dynamic regret. However, this configuration has an unpleasant dependence on the unknown quantity $P_T$, which will be removed in the next part.
\end{myRemark}

\section{Parameter-Free BGD}
\label{sec:parameter-free-BGD}
From Theorems~\ref{thm:dynamic-regret-one-point} and~\ref{thm:dynamic-regret-two-point}, we observe that the optimal parameter configurations of BGD algorithm require to know the path-length $P_T$ in advance, which is generally unknown. In this section, we develop a parameter-free algorithm to address this limitation.

The fundamental obstacle in obtaining universal dynamic regret guarantees is that the path-length $P_T$ remains unknown even after all iterations, since the comparator sequence $\u_1,\ldots,\u_T$ can be chosen arbitrarily from the feasible set. Therefore, the well-known doubling trick~\citep{JACM'97:doubling-trick} is not applicable to remove the dependence on the unknown path-length. Another possible technique to overcome this difficulty is to grid search the optimal parameter by maintaining multiple learning rates in parallel and using expert-tracking algorithms to combine predictions and track the best parameter~\citep{NIPS'16:MetaGrad}. However, it is infeasible to directly apply this method to bandit convex optimization because of the inherent difficulty of bandit setting --- it is only allowed to query the function value \emph{once} at each iteration.

To address this issue, we need a closer investigation of dynamic regret analysis of BCO problems.  Taking the one-feedback model as an example, the expected dynamic regret can be decomposed into three terms,
\begin{equation}
  \begin{split}
  \label{eq:regret-decompose}
    & \E\left[\sum_{t=1}^T f_t(\x_t)\right] - \sum_{t=1}^T f_t(\u_t)\\
  = & \underbrace{\E\left[\sum_{t=1}^T \Big(\hat{f}_t(\y_t) - \hat{f}_t(\mathbf{v}_t)\Big)\right]}_{\mathtt{term~(a)}} + \underbrace{\E\left[\sum_{t=1}^T \Big(f_t(\x_t) - \hat{f}_t(\y_t)\Big)\right]}_{\mathtt{term~(b)}} + \underbrace{\E\left[\sum_{t=1}^T \Big(\hat{f}_t(\v_t) - f_t(\u_t)\big)\right]}_{\mathtt{term~(c)}},
  \end{split}
\end{equation}
where $\v_1,\ldots,\v_T$ is the scaled comparator sequence set as $\v_t = (1-\alpha)\u_t$. It turns out that term~(b) and term~(c) can be bounded by $2L\delta T$ and $(L\delta + L \alpha R)T$ respectively without involving the unknown path-length, and the rigorous argument can be found in~\eqref{eq:diff-1} and~\eqref{eq:diff-2} of Section~\ref{sec:proof-BGD-one}. Hence, it suffices to design parameter-free algorithms to optimize term~(a), i.e., the dynamic regret of the smoothed loss function $\hat{f}_t$. 

However, it remains infeasible to maintain multiple learning rates for optimizing dynamic regret of $\hat{f}_t$. Suppose there are in total $N$ experts where each expert is associated with a learning rate (step size), then at iteration $t$, expert-algorithms will require the information of $\nabla \hat{f}_t(\y_t^1), \nabla \hat{f}_t(\y_t^2),\ldots,\nabla \hat{f}_t(\y_t^N)$ to perform the bandit gradient descent. This  necessitates to query $N$ function values of original loss $f_t$, which is prohibited in bandit convex optimization. 

Fortunately, we discover that the expected dynamic regret of $\hat{f}_t$ can be upper bounded by that of a linear function, as demonstrated in the following proposition.
\begin{myProp}\label{prop:ub-smooth-loss}
\begin{equation}
	\label{eq:ub-smooth-loss}
	\E[\hat{f}_t(\y_t) - \hat{f}_t(\mathbf{v}_t)] \leq \E[\inner{\tilde{g}_t}{\y_t - \mathbf{v}_t}].
\end{equation}
\end{myProp}


This feature motivates us to design the following \emph{surrogate loss} function $\ell_t: (1-\alpha)\X \mapsto \R$,
\begin{equation}
  \label{eq:surrogate-loss}
  \ell_t(\y) = \langle \gt_t,\y - \y_t \rangle,
\end{equation}
which can be regarded as a linearization of smoothed function $\hat{f}_t$ on the point $\y_t$ in terms of expectation. Furthermore, the surrogate loss function enjoys the following two properties.
\begin{myProperty}
  \label{property:surrogate-loss-1}$\forall \y \in (1-\alpha)\X$, $\nabla \ell_t(\y) = \gt_t$.
\end{myProperty}

\begin{myProperty}
  \label{property:surrogate-loss-2}
  $\forall \v \in (1-\alpha)\X$,
  \begin{equation}
  	\label{eq:property2}
  	\E[\hat{f}_t(\y_t) - \hat{f}_t(\v)] \leq \E[\ell_t(\y_t) -\ell_t(\v)].
  \end{equation}
\end{myProperty}
Property~\ref{property:surrogate-loss-1} follows from the definition of surrogate loss, and Proposition~\ref{prop:ub-smooth-loss} immediately implies Property~\ref{property:surrogate-loss-2}. These two properties are simple yet quite useful, and they together make the grid search feasible in bandit convex optimization. Concretely speaking,
\begin{compactitem}
	\item Property~\ref{property:surrogate-loss-1} implies that we can now initialize $N$ experts to perform the bandit gradient descent \emph{over the surrogate loss} where each expert is associated with a specific learning rate, since all the gradients $\nabla \ell_t(\y_t^1), \nabla \ell_t(\y_t^2),\ldots,\nabla \ell_t(\y_t^N)$ essentially equal to $\tilde{g}_t$, which can be obtained by querying the function value of $f_t$ only once.
	\item Property~\ref{property:surrogate-loss-2} guarantees the expected dynamic regret of smoothed functions $\hat{f}_t$'s is upper bounded by that of the surrogate loss $\ell_t$'s.
\end{compactitem}  
Consequently, we propose to optimize surrogate loss $\ell_t$ instead of original loss $f_t$ (or its smoothed version $\hat{f}_t$). We note that the idea of constructing surrogate loss for maintaining multiple learning rates is originally proposed by~\citet{NIPS'16:MetaGrad} but for different purposes. They construct a quadratic upper bound for original loss $f_t$ as surrogate loss, with the aim to adapt to the potential curvature of online functions in full-information online convex optimization. In this paper, we design the surrogate loss as linearization of smoothed function $\hat{f}_t$ in terms of expectation, to make the grid search of optimal parameter doable in bandit convex optimization. To the best of our knowledge, this is the first time to optimize surrogate loss for maintaining multiple learning rates in \emph{bandit} setup.

In the following, we describe the design details of parameter-free algorithms for the one-point feedback model, and present configurations of BCO with two-point feedback model later (in Section~\ref{sec:configuration-two-point}). 

In the one-point feedback model, the optimal step size is $\eta^*= \sqrt{7R^2 + R P_T}/(dCT^{3/4})$, whose value is unavailable due to the unknown path-length $P_T$. Nevertheless, we confirm 
\begin{equation}
	\label{eq:possible-range}
	\frac{\sqrt{7}R}{dCT^{3/4}}\leq \eta^* \leq \frac{\sqrt{7+2T}R}{dCT^{3/4}}
\end{equation}
always holds from the non-negativity and boundedness of the path-length ($0\leq P_T \leq 2RT$). Hence, we first construct the following pool of candidate step sizes $\H$ to discretize the range of optimal parameter in~\eqref{eq:possible-range},
\begin{equation}
  \label{eq:ss-pool}
  \H = \left\{ \eta_i =  2^{i-1}\frac{\sqrt{7}R}{dCT^{3/4}} \big \vert \ i = 1,\ldots, N \right\},
\end{equation}
where $N = \lceil \frac{1}{2}\log_2 (1 + 2T/7) \rceil + 1$. The above configuration ensures there exists an index $k\in \{1,\ldots,N-1\}$ such that $\eta_k \leq \eta^* \leq \eta_{k+1}=2\eta_k$. More intuitively, there is a step size in the pool $\mathcal{H}$ that is not optimal but sufficiently close to $\eta^*$. Next, we instantiate $N$ expert-algorithms, where the $i$-th expert is a BGD algorithm with parameters $\eta_i \in \H$ and $\delta = T^{-1/4}$. Finally, we adopt an expert-tracking algorithm as the meta-algorithm to combine predictions from all the experts to produce the final decision. Owing to nice theoretical guarantees of the meta-algorithm, dynamic regret of final decisions is comparable to that of the best expert, i.e., the expert-algorithm with near-optimal step size.

\begin{algorithm}[!t]
   \caption{\textsc{PBGD}: Meta-algorithm}
   \label{alg:BCO-meta}
\begin{algorithmic}[1]
   \REQUIRE time horizon $T$, the pool of candidate step sizes $\H$, learning rate of the meta-algorithm $\epsilon$
   \STATE {Run expert-algorithms~\eqref{eq:expert-algorithm} with different step sizes simultaneously}
   \STATE Initialize the weight of each expert as \[
   		w_1^i = \frac{N + 1}{N} \cdot \frac{1}{i(i+1)},\quad \forall i\in[N]
   \]
   \FOR{$t=1$ {\bfseries to} $T$}
      \STATE {Receive $\y_t^i$ from each expert $i \in [N]$}
      \STATE {Obtain $\y_t = \sum_{i \in [N]} w_{t}^i \y_t^i$}
      \STATE {Submit $\x_t = \y_t + \delta \s_t$ and incur loss $f_{t}(\x_t)$}
      \STATE {Compute gradient estimator $\tilde{g}_t$ by~\eqref{eq:gradient-estimator-one-point}}
      \STATE {Construct surrogate loss $\ell_t(\cdot)$ as \eqref{eq:surrogate-loss}}
      \STATE {Update the weight of each expert $i \in [N]$ by $$w_{t+1}^{i} = \frac{w_{t}^{i}\exp(-\epsilon \ell_{t}(\y^i_{t}))}{\sum_{i \in [N]}w_t^i\exp(-\epsilon \ell_{t}(\y^i_t))}$$}
      \STATE {Send the gradient estimator $\tilde{g}_t$ to each expert}
    \ENDFOR
\end{algorithmic}
\end{algorithm}

We present descriptions for expert-algorithm and meta-algorithm of \textsc{PBGD} as follows.

\paragraph{Expert-algorithm.} For each candidate step size from the pool $\mathcal{H}$, we initialize an expert, and the expert $i \in [N]$ performs the online gradient descent over the surrogate loss defined in~\eqref{eq:surrogate-loss}, 
\begin{equation}
  \label{eq:expert-algorithm}
    \y_{t+1}^i = \mbox{Proj}_{(1-\alpha) \X} [\y_t^i - \eta_i \nabla \ell_t(\y_t^i)] = \mbox{Proj}_{(1-\alpha) \X}[\y_t^i - \eta_i \tilde{g}_t],
\end{equation}
where $\eta_i$ is the step size of the expert $i$, shown in~\eqref{eq:ss-pool}.

The above update procedure once again demonstrates the necessity of constructing the surrogate loss. Due to the nice property of surrogate loss (Property~\ref{property:surrogate-loss-1}), at each iteration, all the experts can perform the \emph{exact} online gradient descent in the same direction $\tilde{g}_t$. By contrast, suppose each expert is conducted over the smoothed loss function $\hat{f}_t$, then at each iteration it requires to query multiple gradients $\nabla \hat{f}_t(\y_t^i)$, or equivalently, to query multiple function values $f_t(\x_t^i)$, which are unavailable in bandit convex optimization.

\paragraph{Meta-algorithm.} To combine predictions returned from various experts, we adopt the exponentially weighted average forecaster algorithm~\citep{book/Cambridge/cesa2006prediction} with nonuniform initial weights as the meta-algorithm, whose input is the pool of candidate step sizes $\H$ in~\eqref{eq:ss-pool} and its own learning rate $\epsilon$. The nonuniform initialization of weights aims to make regret analysis tighter, which will be clear in the proof. Algorithm~\ref{alg:BCO-meta} presents detailed procedures. Note that the meta-algorithm itself does not require any prior information of the unknown path-length $P_T$.

The meta-algorithm in Algorithm~\ref{alg:BCO-meta}, together with the expert-algorithm~\eqref{eq:expert-algorithm}, gives \textsc{PBGD} (short for \textit{Parameter-free Bandit Gradient Descent}). The following theorem states the dynamic regret of the proposed \textsc{PBGD} algorithm.
\begin{myThm}
\label{thm:final-order-with-meta}
Under Assumptions~\ref{assum:bounded-region},~\ref{assum:bounded-func-value}, and~\ref{assum:lipschitz-continuity}, with a proper setting of the pool of candidate step sizes $\H$ and the learning rate $\epsilon$, \textsc{PBGD} algorithm enjoys the following expected dynamic regret,
\begin{compactitem}
  \item One-Point Feedback Model: $O\big(T^{\frac{3}{4}}(1+P_T)^{\frac{1}{2}}\big)$;
  \item Two-Point Feedback Model: $O\big(T^{\frac{1}{2}}(1+P_T)^{\frac{1}{2}}\big)$.
\end{compactitem}
The above results hold universally for \emph{any} feasible comparator sequence $\u_1,\ldots,\u_T \in \X$.
\end{myThm}
\begin{myRemark}
  Theorem~\ref{thm:final-order-with-meta} shows that the dynamic regret can be improved from $O\big(T^{\frac{3}{4}}(1 + P_T)^{\frac{1}{2}}\big)$ to $O\big(T^{\frac{1}{2}}(1 + P_T)^{\frac{1}{2}}\big)$ when it is allowed to query two points at each iteration. The attained dynamic regret (though in expectation) of BCO with two-point feedback, surprisingly, is in the same order with that of the full-information setting~\citep{NIPS'18:Zhang-Ader}. This extends the claim argued by~\citet{conf/colt/AgarwalDX10} \emph{knowing the value of each loss function at two points is almost as useful as knowing the value of each function everywhere} to dynamic regret analysis. Furthermore, we will show that the obtained dynamic regret for the two-point feedback model is minimax optimal in the next section.
\end{myRemark}

\section{Lower Bound and Extensions}
\label{sec:lower-bound-extension}
In this section, we investigate the attainable dynamic regret for BCO problems, and then extend our algorithm to an anytime version, that is, an algorithm without requiring the time horizon in advance. Furthermore, we study the adaptive regret for BCO problems, another measure for online learning in non-stationary environments.

\subsection{Lower Bound}
We have the following minimax lower bound of universal dynamic regret for BCO problems.
\begin{myThm}
\label{thm:lower-bound}
For any $\tau \in [0, 2RT]$, there exists a comparator sequence $\u_1,\ldots,\u_T \in \X$ satisfying Assumption~\ref{assum:bounded-region} whose path-length $P_T$ is less than $\tau$, and a sequence of functions satisfying Assumption~\ref{assum:lipschitz-continuity}, such that for any algorithm designed for BCO with one-/two-point feedback who returns $\x_1,\ldots,\x_T$,
\begin{equation}
  \label{eq:lower-bound}
  \sum_{t=1}^{T} f_t(\x_t) - \sum_{t=1}^{T} f_t(\u_t) \geq C\cdot dL\sqrt{(R^2 + R \tau)T},
\end{equation}
where $C$ is a positive constant independent of $T$.
\end{myThm}

The proof is detailed in Appendix~\ref{sec:appendix-lower-bound}. From the above lower bound and the upper bounds in Theorem~\ref{thm:final-order-with-meta}, we know that our dynamic regret for the two-point feedback model is  optimal, while the rate for one-point feedback model remains sub-optimal, where the desired rate is of order $O(T^{3/4}(1+P_T)^{1/4})$ as demonstrated in Remark~\ref{remark:one-point}. Note that the desired bound does not contradict with the minimax lower bound, since $O(T^{3/4}(1+P_T)^{1/4}) = O(T^{1/2} T^{1/4}(1+P_T)^{1/4})$ is larger than the $\Omega(T^{1/2}(1+P_T)^{1/2})$ lower bound by noticing that $P_T = o(T)$. 

Our attained $O(T^{3/4}(1+P_T)^{1/2})$ dynamic regret exhibits a square-root dependence on the path-length, and it will become vacuous when $P_T \geq \sqrt{T}$, though the path-length is typically small. The challenge is that the grid search technique cannot support to approximate the optimal perturbation parameter $\delta^*$ which is also dependent on $P_T$. Otherwise, we have to query the function more than once at each iteration. We will investigate a sharper bound for BCO with one-point feedback in the future. 

\begin{myRemark}
The lower bound holds even all the functions $f_t$'s are strongly convex and smooth in BCO with one-point feedback. This is to be contrasted with that in the full-information setting. The reason is that the minimax static regret of BCO with one-point feedback can neither benefit from strongly convexity nor smoothness~\citep{conf/colt/Shamir13}. This implies the inherent difficulty of learning with bandit feedback.
\end{myRemark}

\subsection{Extension to Anytime Algorithm}
Notice that the proposed \textsc{PBGD} algorithm requires the time horizon $T$ as an input, which is not available in advance. We remove the undesired dependence and develop an \emph{anytime} algorithm. 

Our method is essentially a standard implementation of the doubling trick~\citep{JACM'97:doubling-trick}. Specifically, the idea is to initialize the interval by $2$, and once the actual number of iterations exceeds the current counts, double the counts and restart the algorithm. So there will be $K=\lfloor \log T\rfloor + 1$ epochs and the $i$-th epoch contains $2^i$ iterations. We have the following regret guarantees for the above anytime algorithm.
\begin{myThm}
\label{thm:anytime}
Under the same conditions with Theorem~\ref{thm:final-order-with-meta}, the anytime version of \textsc{PBGD} enjoys the following expected dynamic regret,
\begin{compactitem}
  \item One-Point Feedback Model: $O\big(T^{\frac{3}{4}}(\log T +P_T)^{\frac{1}{2}}\big)$;
  \item Two-Point Feedback Model: $O\big(T^{\frac{1}{2}}(\log T +P_T)^{\frac{1}{2}}\big)$.
\end{compactitem}
The above results hold universally for \emph{any} feasible comparator sequence $\u_1,\ldots,\u_T \in \X$.
\end{myThm}

We take the one-point feedback model as an example and provide a brief analysis as follows. Actually, by the strategy of doubling trick, we can bound the dynamic regret of the anytime algorithm by
\begin{align*}
	  {} & \sum\nolimits_{i=1}^K T_i^{\frac{3}{4}} (1+P_i)^{\frac{1}{2}} \leq \sqrt{\sum\nolimits_{i=1}^K T_i^{\frac{3}{2}}}\sqrt{\sum\nolimits_{i=1}^K (1+P_i)}\\
	= {} & \sqrt{\sum\nolimits_{i=1}^K 2^{\frac{3i}{2}}}\sqrt{\log T + P_T} = O\big(T^{\frac{3}{4}}(\log T + P_T)^{\frac{1}{2}}\big).
\end{align*}
Compared with the $O(T^{3/4}(1+P_T)^{1/2})$ rate of the original PBGD algorithm, we observe that an extra $\log T$ term is suffered due to the anytime demand.

\subsection{Adaptive Regret}
\label{sec:adaptive}
In this part, we investigate the adaptive regret. Following the seminal work of~\citet{ICML'09:Hazan-adaptive}, we define the expected adaptive regret for BCO as
\begin{equation*}
  \mathbb{E}[{\AReg}_T] = \max\limits_{[q,s]\subseteq [T]} \left(\mathbb{E}\left[\sum_{t=q}^{s}f_t(\x_t)\right]-\min\limits_{\x\in\mathcal{X}}\sum_{t=q}^{s}f_t(\x)\right).
\end{equation*}
We note that, in the full-information setting, a stronger version of adaptive regret named \emph{strongly adaptive regret} is introduced by~\citet{ICML'15:Daniely-adaptive}. However, they prove that it is impossible to achieve meaningful strongly adaptive regret in bandit settings, so we focus on the notion defined by~\citet{ICML'09:Hazan-adaptive}.

To minimize the above measure, we propose an algorithm called Minimizing Adaptive regret in Bandit Convex Optimization (\textsc{MABCO}). Our algorithm follows a similar framework used in the Coin Betting for Changing Environment (CBCE) algorithm~\citep{AISTATS'17:coin-betting-adaptive}, which achieves the state-of-the-art adaptive regret in the full-information setting. However, we note that a direct reduction of CBCE algorithm to the bandit setting requires to query the loss function multiple times at each iteration, which is invalid in the bandit feedback model. To address this difficulty, similar to \textsc{PBGD} we introduce a new surrogate loss function, which can be constructed by only using the one-point or two-point function values. We provide algorithmic details and proofs of theoretical results in Appendix~\ref{sec:appendix-adaptive-regret}. 

\begin{myThm}
\label{thm:bandit-adaptive-regret}
With a proper setting of surrogate loss functions and parameters, the proposed \textsc{MABCO} algorithm enjoys the following expected adaptive regret,
\begin{compactitem}
  \item One-Point Feedback Model: $O\big(T^{\frac{3}{4}}(\log T)^{\frac{1}{4}}\big)$;
  \item Two-Point Feedback Model: $O\big(T^{\frac{1}{2}}(\log T)^{\frac{1}{2}}\big)$.
\end{compactitem}
\end{myThm}

Note that we cannot hope for an adaptive regret that is better than the static regret. The adaptive regret in Theorem~\ref{thm:bandit-adaptive-regret} matches $O(T^{3/4})$ and $O(T^{1/2})$ static regret bounds for the one-point~\citep{conf/soda/FlaxmanKM05} and two-point~\citep{conf/colt/AgarwalDX10} feedback models, up to logarithmic factors.

\section{Analysis of BGD Algorithm}
\label{sec:analysis-BGD}
In this section, we provide the proofs of theoretical guarantees for the BGD algorithm including Theorem~\ref{thm:dynamic-regret-one-point} (one-point feedback model) and Theorem~\ref{thm:dynamic-regret-two-point} (two-point feedback model).

Before presenting rigorous proofs, we first highlight the main idea and procedures of the argument as follows.
\begin{compactitem}
  \item[(1)] Guarantee that for any $t \in [T]$, $\x_t$ is a feasible point in $\X$, because the projection in Algorithm~\ref{alg:BCO} is over $\y_t$ instead of $\x_t$.
  \item[(2)] Analyze the dynamic regret of the smoothed functions $\hat{f}_1,\ldots,\hat{f}_T$ in terms of a certain comparator sequence.
  \item[(3)] Check the gap between the dynamic regret of the smoothed functions $\hat{f}_1,\ldots,\hat{f}_T$ and that of the original functions $f_1,\ldots,f_T$.
\end{compactitem}

\subsection{Proof of Theorem~\ref{thm:dynamic-regret-one-point}}
\label{sec:proof-BGD-one}
\begin{proof}
Notice that the projection in Algorithm~\ref{alg:BCO} only guarantees that $\y_t$ is in a slightly smaller set $(1-\alpha)\X$, so we first need to prove that $\forall t \in [T]$, $\x_t$ is a feasible point in $\X$. This is convinced by Lemma~\ref{lemma:project-ball}, since we know that $\delta \leq \alpha r$ from the parameter setting ($\alpha = \delta/r$).

Next, as demonstrated in~\eqref{eq:regret-decompose}, the expected dynamic regret can be decomposed into three terms. So we will bound the three terms separately.

The term (a) is essentially the dynamic regret of the smoothed functions. In the one-point feedback model, the gradient estimator is set according to~\eqref{eq:gradient-estimator-one-point}, and we know that $\E[\tilde{g}_t] = \nabla \hat{f}_t(\y_t)$ due to Lemma~\ref{lemma:gradient-est}. Therefore, the procedure of $\y_{t+1} = \mbox{Proj}_{(1-\alpha)\X} [\y_t - \eta\tilde{g}_t]$ is actually the randomized online gradient descent over the smoothed function $\hat{f}_t$. So term (a) can be upper bound by using Theorem~\ref{thm:dynamic-regret-random-OGD}.
\begin{equation}
  \label{eq:dynamic-regret-hatf}
  \mathtt{term (a)} \overset{\eqref{eq:dynamic-regret-random-OGD}}{\leq} \frac{7\tilde{D}^2 + \tilde{D}\tilde{P_T}}{4\eta} + \frac{\eta \tilde{G}^2T}{2} \leq \frac{7R^2 + RP_T}{4\eta} + \frac{ \eta d^2 C^2T}{2\delta^2},
\end{equation}
where $\tilde{P}_T = \sum_{t=2}^{T} \norm{\v_{t-1} - \v_t}_2 = (1-\alpha) P_T$, $\tilde{D} = (1-\alpha)R \leq R$ and $\tilde{G} = d C/\delta$ by noticing
\begin{equation}
  \label{eq:gradient-upper-one-point}
   \norm{\tilde{g}_t}_2 \leq \left\|\frac{d}{\delta}f_t(\y_t + \delta \s_t)\mathbf{s}_t\right\|_2 \overset{\eqref{eq:func-value}}{\leq} d C/\delta, \ \forall t\in [T].
\end{equation}
Now, it suffices to bound term (b) and term (c). By Assumption~\ref{assum:lipschitz-continuity} and Lemma~\ref{lemma:diff-f-and-hf}, we have
\begin{equation}
  \label{eq:diff-1}
  \mathtt{term (b)} = \E\left[\Big(\sum_{t=1}^{T} f_t(\x_t) - f_t(\y_t) + f_t(\y_t) - \hat{f}_t(\y_t)\Big)\right] \leq 2L\delta T.
\end{equation}
And term (c) can be bounded by
\begin{equation}
    \label{eq:diff-2}
  \begin{split}
  \mathtt{term (c)} & \leq \E\left[ \sum_{t=1}^T \abs{\hat{f}_t(\mathbf{v}_t) - f_t(\u_t)} \right] = \E\left[ \sum_{t=1}^T \Big(\abs{\hat{f}_t(\mathbf{v}_t) - f_t(\mathbf{v}_t)} + \abs{f_t(\mathbf{v}_t) - f_t(\u_t)}\Big)\right] \\
  & \leq \E\left[\sum_{t=1}^{T} (L\delta + L \norm{\v_t - \u_t}_2)\right] \leq \E\left[\sum_{t=1}^{T} (L\delta + L \alpha R) \right] = (L + \frac{LR}{r})\delta T
  \end{split}
\end{equation}
where the second inequality holds due to Lemma~\ref{lemma:diff-f-and-hf} and Assumption~\ref{assum:lipschitz-continuity}.

By combining upper bounds of three terms in~\eqref{eq:dynamic-regret-hatf},~\eqref{eq:diff-1} and~\eqref{eq:diff-2}, we obtain the dynamic regret of the original function $f_t$ over the comparator sequence of $\u_1,\ldots,\u_T$,
\begin{align}
       {} & \E\left[\sum_{t=1}^T f_t(\x_t)\right] - \sum_{t=1}^T f_t(\u_t)  \nonumber\\
    =  {} & \mathtt{term~(a)} + \mathtt{term~(b)} + \mathtt{term~(c)} \nonumber\\
  \leq {} & \frac{7R^2 + RP_T}{4\eta} + \frac{ \eta d^2 C^2T}{2\delta^2} + 2L\delta T + (L\delta + L \alpha R)T \nonumber\\
  \leq {} & \frac{7R^2 + RP_T}{4\eta} + \frac{ \eta d^2 C^2T}{2\delta^2} + \big(3L + \frac{LR}{r}\big)\delta T \label{eq:final-detail-Lipschitz}\\
  = {} & O\Big((1 + P_T)^{\frac{1}{4}}T^{\frac{3}{4}}\Big) \nonumber,
\end{align}
where~\eqref{eq:final-detail-Lipschitz} follows from the setting of $\alpha = \delta/r$; the last equation is obtained by the AM-GM inequality via optimizing values of $\eta$ and $\delta$. The optimal parameter configuration is
\begin{align*}
  \begin{cases}
  \delta^* & = \left(\frac{7R^2 + P_T}{T}\right)^{\frac{1}{4}}2^{-\frac{1}{4}}\big(dC/(3L+LR/r)\big)^{\frac{1}{2}}, \\
  \eta^* & = \left(\frac{7R^2 + P_T}{T}\right)^{\frac{3}{4}}2^{-\frac{3}{4}}\big(dC(3L+LR/r)\big)^{-\frac{1}{2}}.
  \end{cases}
\end{align*}
\end{proof}

\subsection{Proof of Theorem~\ref{thm:dynamic-regret-two-point}}
\label{sec:analysis-dynamic-two-point}
\begin{proof}
In the two-point feedback model, the gradient estimator is constructed according to~\eqref{eq:gradient-estimator-two-point}, whose norm can be upper bounded as follows,
\begin{equation}
  \label{eq:gradient-upper-two-point}
  \begin{split}
  \norm{\tilde{g}_t}_2 &= \frac{d}{2\delta} \norm{(f_t(\y_t + \delta\s_t) - f_t(\y_t - \delta\s_t))\s_t}_2 \\
  & = \frac{d}{2\delta} \abs{f_t(\y_t + \delta\s_t) - f_t(\y_t - \delta\s_t)} \\
  & \overset{\eqref{eq:lipschitz}}{\leq} \frac{dL}{2\delta} \norm{2\delta\s_t}_2 = Ld,
  \end{split}
\end{equation}
where in the last inequality, we utilize the Lipschitz property due to Assumption~\ref{assum:lipschitz-continuity}. Hence, $\tilde{G} = \sup_{t \in [T]} \norm{\tilde{g}_t}_2 = Ld$. We remark that by contrast with that in the one-point feedback model as shown in~\eqref{eq:gradient-upper-one-point}, the upper bound of gradient norm $\tilde{G}$ here is \emph{independent} of the $1/\delta$, which leads to a substantially improved regret bound.

Meanwhile, by exploiting the Lipschitz property, we have
\begin{equation}
  \label{eq:two-point-Lipschitz}
  f_t(\y_t + \delta\s_t) \leq f_t(\y_t) + L\norm{\delta\s_t}_2 = f_t(\y_t) + \delta L,
\end{equation}
and similar result holds for $f_t(\x_t - \delta\s_t)$. We can thus bound the expected regret as follows,
\begin{align}
     & \E\left[\sum_{t=1}^T \frac{1}{2}\big(f_t(\y_t + \delta \s_t) + f_t(\y_t - \delta \s_t)\big)\right] - \sum_{t=1}^T f_t(\u_t) \nonumber \\
    \overset{\eqref{eq:two-point-Lipschitz}}{\leq} & \E\left[\sum_{t=1}^T f_t(\y_t)\right] + \delta LT - \sum_{t=1}^T f_t(\u_t) \nonumber \\
    =    & \E\left[\sum_{t=1}^T \hat{f}_t(\y_t) - \sum_{t=1}^T \hat{f}_t(\mathbf{v}_t)\right] + \delta LT + \E\left[\sum_{t=1}^T f_t(\y_t) - \hat{f}_t(\y_t)\right] + \left[\sum_{t=1}^T \Big(\hat{f}_t(\mathbf{v}_t) - f_t(\u_t)\Big)\right] \nonumber \\
    \leq & \frac{7R^2 + RP_T}{4\eta} + \frac{\eta L^2 d^2}{2}T + \big(3L + \frac{LR}{r}\big)\delta T \label{eq:two-point-step4}\\
    = & O\Big((1 + P_T)^{\frac{1}{2}}T^{\frac{1}{2}}\Big) \label{eq:two-point-step5}
   \end{align}
The core characteristic of analysis of the two-point feedback model lies in the second term of~\eqref{eq:two-point-step4}, which is independent of $1/\delta$, and thus is much smaller than that of~\eqref{eq:final-detail-Lipschitz}. This owes to the benefit of the gradient estimator evaluated by two points at each iteration. Notice that~\eqref{eq:two-point-step5} is obtained by setting $\delta = 1/\sqrt{T}$ and $\eta = \sqrt{(7R^2 + RP_T)/(2L^2d^2 T)}$.
\end{proof}

\section{Analysis of PBGD Algorithm}
\label{sec:analysis-PBGD}
In this section, we provide the proofs of theoretical guarantees for the PBGD algorithm including Proposition~\ref{prop:ub-smooth-loss} and Theorem~\ref{thm:final-order-with-meta} (both one-point and two-point feedback models). Besides, we present the algorithmic details for BCO with two-point feedback.

\subsection{Proof of Proposition~\ref{prop:ub-smooth-loss}}
\begin{proof} First, notice that from the convexity of the smoothed function $\hat{f}$, we have
\begin{equation}
  \label{eq:lemma-proof-1}
  \hat{f}_t(\y_t) - \hat{f}_t(\v_t) \leq \inner{\nabla \hat{f}_t(\y_t)}{\y_t - \v_t} = \inner{\nabla \hat{f}_t(\y_t) - \tilde{g}_t}{\y_t - \v_t} + \inner{\tilde{g}_t}{\y_t - \v_t}.
\end{equation}

Besides, similar to the argument of~\citet{conf/soda/FlaxmanKM05}, let $\xi_t = \nabla \hat{f}_t(\y_t) - \tilde{g}_t$, then $\E[\xi_t|\x_1,f_1,\ldots,\x_t,f_t] = 0$ due to Lemma~\ref{lemma:gradient-est}. Thus, for any \emph{fixed} $\x \in \X$, we have
\begin{equation}
  \label{eq:lemma-proof-2}
  \E[\xi_t^\T \x ] = \E[\E[\xi_t^\T \x | \x_1,f_1,\ldots,\x_t,f_t]] = \E[\E[\xi_t | \x_1,f_1,\ldots,\x_t,f_t]^\T\x ] = 0,
\end{equation}
which implies $\E[\inner{\nabla \hat{f}_t(\y_t) - \tilde{g}_t}{\y_t - \v_t}] = 0$ since the comparator sequence is assumed to be chosen by an oblivious adversary. 
\end{proof}

\subsection{Proof of Theorem~\ref{thm:final-order-with-meta} (One-Point Feedback Model)}
\label{sec:appendix-proof-meta}
\begin{proof}
As shown in~\eqref{eq:regret-decompose}, the expected dynamic regret can be decomposed into three terms,
\begin{equation*}
  \begin{split}
    & \E\left[\sum_{t=1}^T f_t(\x_t)\right] - \sum_{t=1}^T f_t(\u_t)\\
  = & \underbrace{\E\left[\sum_{t=1}^T \Big(\hat{f}_t(\y_t) - \hat{f}_t(\mathbf{v}_t)\Big)\right]}_{\mathtt{term~(a)}} + \underbrace{\E\left[\sum_{t=1}^T \Big(f_t(\x_t) - \hat{f}_t(\y_t)\Big)\right]}_{\mathtt{term~(b)}} + \underbrace{\E\left[\sum_{t=1}^T \Big(\hat{f}_t(\v_t) - f_t(\u_t)\big)\right]}_{\mathtt{term~(c)}}.
  \end{split}
\end{equation*}
From the analysis of BGD, shown in~\eqref{eq:diff-1} and~\eqref{eq:diff-2}, we know that the term~(b) and term~(c) are at most $2L\delta T$ and $(L\delta + L \alpha R)T$ respectively. Hence, it suffices to bound term (a). Since term (a) is over the original loss functions, while the algorithm performs over the surrogate loss function, we need to establish their relationship. Actually, Proposition~\ref{prop:ub-smooth-loss} implies that the term (a) can be upper bounded by 
\begin{equation}
  \label{eq:bound-term-a}
  \mathtt{term~(a)} \leq \E\Bigg[\underbrace{\sum_{t=1}^{T}\big(\ell_t(\y_t) - \ell_t(\v_t)\big)}_{:=D_T}\Bigg].
\end{equation}
Notably, the quantity in the expectation is essentially the dynamic regret over the surrogate loss and can be divided as
\begin{align}
  \label{eq:surro-regret-decompose}
  D_T = \underbrace{\sum_{t=1}^{T}\big(\ell_t(\y_t) - \ell_t(\y_t^k)\big)}_{\meta} + \underbrace{\sum_{t=1}^{T}\big(\ell_t(\y_t^k) - \ell_t(\v_t)\big)}_{\expert},
\end{align}
where $\y_1^k,\ldots,\y_T^k$ is the prediction sequence returned by the expert $k$. Note that the above decomposition holds for any expert $k \in [N]$. In the following, we will bound the expert-regret and meta-regret respectively.

First, we examine the expert-regret. The regret decomposition~\eqref{eq:surro-regret-decompose} holds for any expert $k \in [N]$, we therefore choose the best expert to obtain a sharp bound. Specifically, due to the boundedness of path-length $P_T$ and the setting of optimal step size $\eta^*$, we can verify that there exists an index $k^*\in \{1,\ldots,N-1\}$ such that $\eta_{k^*} \leq \eta^* \leq \eta_{k^*+1} = 2\eta_{k^*}$ with
\begin{equation}
  \label{eq:k-one-point}
  k^* \leq \Big\lceil \frac{1}{2}\log_2\big(1 + \frac{P_T}{7R}\big)\Big\rceil + 1.
\end{equation}

In other words, the expert $k^*$ is the best expert in the pool in the sense that it has a near-optimal step size $\eta_{k^*}$ to approximate the unknown step size $\eta_*$. Since each expert performs the deterministic online gradient descent over surrogate loss, we can apply the existing dynamic regret guarantee of OGD (Theorem~\ref{thm:dynamic-regret-OGD}) and obtain that
\begin{equation}
  \label{eq:one-point-expert}
  \begin{split}
   \expert \leq {} & \frac{7R^2 + RP_T}{4\eta_{k^*}} + \frac{\eta_{k^*} \tilde{G}^2 T}{2}  \\
   \leq {} & \frac{7R^2 + RP_T}{2\eta^*} + \frac{\eta^* d^2 C^2 T}{2\delta^2} \\
   = {} & \frac{3\sqrt{2}}{4}dCT^{\frac{3}{4}}\sqrt{7R^2 + RP_T},
  \end{split}
\end{equation}
where the first inequality follows from the dynamic regret guarantee of OGD, second inequality holds due to $\eta_{k^*} \leq \eta^* \leq 2\eta_{k^*}$, and the last one holds due to the setting of the optimal step size $\eta^*=((7R^2 + RP_T)/T)^{3/4}$ and the perturbation parameter $\delta = T^{-1/4}$.

Next, we bound the meta-regret. Note that the meta-algorithm is essentially the exponentially weighted average forecaster with nonuniform initial weights. Therefore, by noticing that the magnitude of surrogate loss $\ell_t$ is at most
\[
  \abs{\ell_t(\y)} = \abs{\langle \gt_t,\y - \y_t \rangle} \leq \norm{\gt_t}_2 \norm{\y - \y_t}_2 \overset{\eqref{eq:bounded-region}}{\leq} 2\tilde{G}R, \ \forall \y\in (1-\alpha)\X, t\in [T],
\]
we can apply the standard regret guarantee of exponentially weighted average forecaster with nonuniform initial weights~\citep[Excercise 2.5]{book/Cambridge/cesa2006prediction} and obtain the following meta-regret bound.
\begin{myLemma}
\label{lemma:meta-algorithm}
For any step size $\epsilon > 0$, we have
\begin{equation*}
  \sum_{t=1}^{T} \ell_t(\y_t) - \min_{i\in [N]} \left(\sum_{t=1}^{T} \ell_t(\y^i_t) + \frac{1}{\epsilon}\ln \frac{1}{w_1^i}\right) \leq 2\epsilon T\tilde{G}^2R^2.
\end{equation*}
Therefore, by setting $\epsilon = \sqrt{1/(2T\tilde{G}^2R^2)}$ to minimize the above upper bound, we obtain
\begin{equation*}
  \sum_{t=1}^{T} \ell_t(\y_t) - \sum_{t=1}^{T} \ell_t(\y^i_t) \leq \tilde{G} R\sqrt{2T}\left(1 + \ln \frac{1}{w_1^i}\right).
\end{equation*}
for any index $i \in [N]$, where $\tilde{G}$ is the magnitude of the gradient estimator.
\end{myLemma}
In particular, the lemma holds for the expert $k^*$, so we have
\begin{equation}
  \label{eq:one-point-meta}
  \begin{split}  
  \meta \leq & \tilde{G} R\sqrt{2T}\left(1 + \ln\frac{1}{w_1^{k^*}}\right) \\
  \leq & \frac{dCR}{\delta}\sqrt{2T}\big(1+2\ln(k^*+1)\big).
  \end{split}
\end{equation}
By combining upper bounds of expert-regret~\eqref{eq:one-point-expert} and meta-regret~\eqref{eq:one-point-meta}, we conclude that the term (a) is at most
\begin{align*}
  \mathtt{term~(a)} \leq \sqrt{2}dCRT^{3/4}\big(1+2\ln(k^*+1) + 3\sqrt{7R^2 + RP_T} /4\big),
\end{align*}
which in conjunction with upper bounds of term (b) and term (c) in~\eqref{eq:diff-1} and~\eqref{eq:diff-2} finally yields the expected dynamic regret bound as follows,
\begin{align*}
       {} & \E\left[\sum_{t=1}^T f_t(\x_t)\right] - \sum_{t=1}^T f_t(\u_t) \\
    =  {} & \mathtt{term~(a)} + \mathtt{term~(b)} + \mathtt{term~(c)}\\
  \leq {} & \mathtt{term~(a)} + 2L\delta T + (L\delta + L \alpha R)T \\
  \leq {} & \sqrt{2}dCRT^{3/4}\big(1+2\ln(k^*+1) + 3\sqrt{7R^2 + RP_T} /4\big)  + (3L + LR/r)T^{3/4}  \\
  = {} & O\big(T^{3/4}(1 + P_T)^{1/2}\big),
\end{align*}
where the last equation makes use of the upper bound of index $k^*$ in~\eqref{eq:k-one-point}.
\end{proof}

\subsection{Proof of Theorem~\ref{thm:final-order-with-meta} (Two-Point Feedback Model)}
\label{sec:configuration-two-point}
In this part, we first present the configuration of the step size pool $\H$ for the two-point feedback model, and then provide the proof of dynamic regret.

In the two-point feedback model, the optimal step size is $\eta^* = \sqrt{\frac{7R^2 + RP_T}{2L^2d^2 T}}$, and we know 
\begin{equation*}
  \label{eq:possible-range-two-point}
  \sqrt{\frac{7R^2}{2L^2d^2 T}} \leq \eta^* \leq \sqrt{\frac{7R^2 + 2R^2T}{2L^2d^2 T}}
\end{equation*}
always holds due to $0\leq P_T \leq 2RT$. Hence, we construct the following pool of candidate step sizes $\H$ as,
\begin{equation*}
  \label{eq:ss-pool-two-point}
  \H = \Big\{ \eta_i = 2^{i-1}\sqrt{\frac{7R^2}{2L^2d^2 T}} \big \vert \ i = 1,\ldots, N \Big\},
\end{equation*}
where $N = \lceil \frac{1}{2} \log_2(1+ \frac{2T}{7}) \rceil + 1$. Based on the configurations, we proceed to present the proof of Theorem~\ref{thm:final-order-with-meta} for the two-point feedback model.
\begin{proof}
The proof is analogous to that of one-point feedback model, where the main differences lie in two quantities: the index of optimal expert $k^*$, and the magnitude of the gradient estimator $\tilde{G}$. In the two-point feedback model, the index of best expert $k*$ is at most
\begin{equation}
  \label{eq:k-two-point}
  k^* \leq \Big\lceil \frac{1}{2}\log_2\big(1 + \frac{P_T}{7R}\big)\Big\rceil + 1
\end{equation}
and the associated step size satisfies that $\eta_{k^*} \leq \eta^* \leq \eta_{k+1}$. Besides, during the analysis of BGD, we have that the magnitude of the gradient estimator $\tilde{G} \leq Ld$, as shown in~\eqref{eq:gradient-upper-two-point}. 

So the expert-regret is upper bounded by 
\begin{equation*}
\begin{split}
   {} & \expert \leq \frac{7R^2 + RP_T}{4\eta_{k^*}} + \frac{\eta_{k^*} \tilde{G}^2 T}{2} \\
   \leq {} & \frac{7R^2 + RP_T}{2\eta^*} + \frac{\eta^* L^2d^2 T}{2\delta^2} \\
   = {} & \frac{3\sqrt{2}}{4}Ld \sqrt{T(7R^2 + RP_T)},
  \end{split}
\end{equation*}
where the last equation is obtained by plugging the parameter setting of $\eta^*$ and $\delta = T^{-1/2}$. Besides, the meta-regret is bounded by
\begin{equation*}
\begin{split}
  \meta \leq {} & \tilde{G} R\sqrt{2T}\left(1 + \ln(1/w_1^{k^*})\right) \\
  \leq {} & LdR\sqrt{2T}\big(1+2\ln(k^*+1)\big).
\end{split}
\end{equation*}
Therefore, by combining upper bounds of meta-regret and expert-regret, we have
\begin{align*}
  \mathtt{term~(a)} \leq LdR\sqrt{2T}\big(1+2\ln(k^*+1)\big) + \frac{3\sqrt{2}}{4}Ld \sqrt{T(7R^2 + RP_T)},
\end{align*}
which in conjunction with upper bounds of term (b) and term (c) in~\eqref{eq:diff-1} and~\eqref{eq:diff-2} finally yields the expected dynamic regret bound as follows,
\begin{align*}
    {}& \E\left[\sum_{t=1}^T f_t(\x_t)\right] - \sum_{t=1}^T f_t(\u_t) \\
    =  {} & \mathtt{term~(a)} + \mathtt{term~(b)} + \mathtt{term~(c)}\\
  \leq {} & \mathtt{term~(a)} + 2L\delta T + (L\delta + L \alpha R)T \\
  \leq {} & LdR\sqrt{2T}\big(1+2\ln(k^*+1)\big) + \frac{3\sqrt{2}}{4}Ld \sqrt{T(7R^2 + RP_T)} + (3L + LR/r)T^{1/2}  \\
  = {} & O\big(T^{1/2}(1 + P_T)^{1/2}\big).
\end{align*}
where the last equation makes use of the upper bound of index $k^*$ in~\eqref{eq:k-two-point}.
\end{proof}

\section{Conclusion and Future Work}
\label{sec:conclusion}
In this paper, we study the bandit convex optimization (BCO) problems in non-stationary environments. We propose the Parameter-free Bandit Gradient Descent (PBGD) algorithm that achieves the state-of-the-art $O(T^{3/4}(1+P_T)^{1/2})$ and $O(T^{1/2}(1+P_T)^{1/2})$ dynamic regret for one-point and two-point feedback models respectively. The regret bounds hold universally for any feasible comparator sequence. Meanwhile, the algorithm does not need to know prior information of the path length, which is unknown but required in previous studies. Furthermore, we demonstrate the regret bound for the two-point feedback model is minimax optimal by establishing the first lower bound for the universal dynamic regret in the bandit convex optimization setup. We extend the algorithm to an anytime version. Besides, we also present the algorithm for BCO problems to optimize the adaptive regret, another measure for non-stationary online learning.

In the future, we will investigate a sharper bound for BCO with one-point feedback. Moreover, we will consider incorporating other properties, like strong convexity and smoothness, to further enhance the dynamic regret for bandit convex optimization.

\bibliography{online_learning}
\bibliographystyle{abbrvnat}

\appendices
\section{Preliminaries}
\label{sec:appendix-lemma}
In this section, we introduce preliminaries for analyzing dynamic regret and adaptive regret of algorithms for BCO problems.

\subsection{Projection Issues}
Notice that we run the algorithm on a slightly smaller set $(1-\alpha)\X$ rather than the original feasible set $\X$, where the shrinkage parameter $\alpha>0$ needs to be sufficiently large so that the decision $\y_t + \delta \s_t$ (and $\y_t - \delta \s_t$) can be guaranteed to locate in $\X$. Consequently, there are some additional terms involved due to the projection over a shrunk set. In the following we provide some lemmas justifying the relationships between the original feasible set and the shrunk set. Note that most of these results can be found in the seminal paper~\citep{conf/soda/FlaxmanKM05}, we provide the proofs for self-containedness.

\begin{myLemma}
\label{lemma:project-ball}
For any feasible point $\x \in (1-\alpha) \mathcal{X}$, the ball of radius $\alpha r$ centered at $\x$ belongs to the feasible set $\X$.
\end{myLemma}

\begin{proof}
The result is originally proved in Observation 3.2 of~\citet{conf/soda/FlaxmanKM05}. The proof is based on the simple observation that
\[
  (1-\alpha) \X + \alpha r \mathbb{B} \subseteq (1-\alpha) \X + \alpha \X = \X
\]
holds since $r \mathbb{B} \subseteq \X$ and $\X$ is convex.
\end{proof}

The following lemma, originally raised in Observation 3.3 of~\citet{conf/soda/FlaxmanKM05}, establishes a bound on the maximum that the function can change in $(1-\alpha)\X$, which essentially acts as an effective Lipschitz condition.

\begin{myLemma}
  \label{lemma:diff-f-and-hf}
  For any $\x \in (1-\alpha) \X$, under Assumption~\ref{assum:lipschitz-continuity}, we have
	\begin{equation}
		\label{eq:diff-f-and-hf}
		\abs{\hat{f}_t(\x) - f_t(\x)} \leq L \delta.
	\end{equation}
\end{myLemma}
\begin{proof}
  Since the smoothed function $\hat{f}_t$ is an average over inputs within $\delta$ of $\x$, the Lipschitz continuity of the function $f_t$ yields the result.
\end{proof}

\subsection{Dynamic Regret}
\label{append:dynamic}
We have following dynamic regret bound for the online gradient descent~\citep{ICML'03:zinkvich}.
\begin{myThm}[Dynamic Regret of OGD]
\label{thm:dynamic-regret-OGD}
Consider the online gradient descent (OGD), which starts with any $\x_1 \in \mathcal{X}$ and performs
\begin{equation*}
  \x_{t+1} = \mbox{Proj}_{\mathcal{X}} [\x_t - \eta \nabla f_t(\x_t)].
\end{equation*}

Suppose the feasible domain $\X$ is bounded, i.e., $\norm{\x - \y}_2 \leq D$ for any $\x,\y \in \X$; meanwhile, the online functions have bounded gradient magnitude, i.e., $\norm{\nabla f_t(\x)}_2 \leq G $ for any $\x\in \X$ and $t\in[T]$. Then, the dynamic regret of OGD is upper bounded by
\begin{equation*}
  \label{eq:dynamic-regret-OGD}
  \sum_{t=1}^{T} f_t(\x_t) - \sum_{t=1}^{T} f_t(\u_t) \leq \frac{7D^2 + DP_T}{4\eta} + \frac{\eta G^2 T}{2},
\end{equation*}
for \emph{any} comparator sequence $\u_1,\ldots,\u_T \in \X$. In above, $P_T$ is its path-length defined as $P_T = \sum_{t=2}^{T} \norm{\u_t - \u_{t-1}}_2$.
\end{myThm}

In the bandit convex optimization setting, we cannot access the true gradient but the unbiased gradient estimation instead. Therefore, we extend Theorem~\ref{thm:dynamic-regret-OGD} to the randomized version for the loss function chosen from adaptive environments as follows.

\begin{myThm}[Expected Dynamic Regret of Randomized OGD]
\label{thm:dynamic-regret-random-OGD}
Consider the following randomized version online gradient descent. The randomized OGD begins with any $\x_1 \in \mathcal{X}$ and performs
\begin{equation}
  \label{eq:random-OGD}
  \x_{t+1} = \mbox{Proj}_{\mathcal{X}} [\x_t - \eta g_t],
\end{equation}
where $\E[g_t|\x_1,f_1,\ldots,\x_t,f_t] = \nabla f_t(\x_t)$ and $\norm{g_t}_2 \leq \tilde{G}$ for some $\tilde{G} > 0$. Then, the expected dynamic regret of OGD is upper bounded by
\begin{equation}
  \label{eq:dynamic-regret-random-OGD}
  \E\left[\sum_{t=1}^{T} f_t(\x_t)\right] - \sum_{t=1}^{T} f_t(\u_t) \leq \frac{7D^2 + DP_T}{4\eta} + \frac{\eta \tilde{G}^2 T}{2},
\end{equation}
for \emph{any} fixed comparator sequence $\u_1,\ldots,\u_T \in \mathcal{X}$.
\end{myThm}

\begin{proof}
Define the function $h_t: \X \rightarrow \R$ by
\begin{equation}
  \label{eq:function-h}
  h_t(\x) = f_t(\x) + \inner{\x}{\xi_t}, \quad \mbox{where}\   \xi_t = g_t - \nabla f_t(\x_t).
\end{equation}

Clearly, $\nabla h_t(\x_t) = \nabla f_t(\x_t) + \xi_t = g_t$. So we can leverage the result of deterministic version OGD in Theorem~\ref{thm:dynamic-regret-OGD} on the function $h_t$ and obtain that
\begin{align}
  \label{eq:for-h}
  \sum_{t=1}^{T} h_t(\x_t) - \sum_{t=1}^{T} h_t(\u_t) \leq \frac{7D^2 + DP_T}{4\eta} + \frac{\eta \tilde{G}^2 T}{2}.
\end{align}

Note that for any \emph{fixed} $\x \in \X$, we have
\begin{equation}
	\label{eq:expectation-oblivious}
	\begin{split}
	\E[h_t(\x)] & = \E[f_t(\x)] + \E[\xi_t^\T \x ]\\
	& = \E[f_t(\x)] + \E[\E[\xi_t^\T \x | \x_1,f_1,\ldots,\x_t,f_t]]\\
	& = \E[f_t(\x)] + \E[\E[\xi_t | \x_1,f_1,\ldots,\x_t,f_t]^\T\x ]\\
	& = \E[f_t(\x)].
	\end{split}
\end{equation}
Therefore, when both the function sequence and comparator sequence are chosen by an oblivious adversary (as specified in Section~\ref{sec:assump}), we can take expectations over both sides of \eqref{eq:for-h} and obtain the desired result. 
\end{proof}

\subsection{Adaptive Regret}
In the full-information setting, we have the following adaptive regret bound for the Coin Betting for Changing Environment (\textsc{CBCE}) algorithm proposed by~\citet{AISTATS'17:coin-betting-adaptive} . 
\begin{myThm}[Adaptive Regret of CBCE{~\citep[Theorem 1]{AISTATS'17:coin-betting-adaptive}}]
\label{thm:saregret}
Consider an OCO problem where at iteration $t$ a learner iteratively select a decision $\x_t \in \X$ and observes a loss function $h_t$. Assume the gradient of all the loss functions are bounded by $G$, the diameter of $\mathcal{X}$ is bounded by $D$, and the function value of $h_t$ lies in $[0,1]$, $\forall t\in[T]$. Then, the CBCE algorithm with the standard OGD algorithm as its expert-algorithm and $h_1,\dots,h_T$ as the input loss functions achieves the following adaptive regret,
\begin{equation*}
\begin{split}
\max\limits_{[q,s] \subseteq [T]} \left(\sum_{t=q}^{s}h_t(\x_t)-\min\limits_{\x\in\mathcal{X}}\sum_{t=q}^{s}h_t(\x)\right)&\leq 15DG\sqrt{T}+8\sqrt{7\log T+5}\sqrt{T}.
\end{split}
\end{equation*}
\end{myThm}
The algorithm above is inefficient in the sense that it requires to query the gradient of the loss function $O(\log t)$ times at iteration $t$. To address this limitation,~\citet{IJCAI:2018:Wang} introduce a surrogate loss function $\ell_t:\mathcal{\mathcal{X}}\mapsto[0,1]$,
\begin{equation*}
  \ell_t(\x)=\frac{1}{2DG}\nabla h_t(\x_t)^{\top}(\x-\x_t)+\frac{1}{2}
\end{equation*}
for which we have $\forall \x\in{\mathcal{X}}$,
\begin{equation}
\label{thm:saregret-relation}
h_t(\x_t)-h_t(\x)\leq-2DG\ell_{t}(\x)+DG=2DG(\ell_{t}(\x_t)-\ell_t({\x})).
\end{equation}
Notice that the inequality \eqref{thm:saregret-relation} implies that, to solve the original problem where the loss functions are $h_1(\cdot),\dots, h_T(\cdot)$, we can deploy CBCE on a new problem where the loss functions are $\ell_1(\cdot),\dots, \ell_T(\cdot)$. The benefits here is that in this way we only need to query the gradient of $h_t$ once at each iteration and the order of the regret bound remains the same. To be more specific, we have the following regret bound.
\begin{myThm}
\label{thm:saregret-relation-cor}
Consider the same learning setting as in Theorem~\ref{thm:saregret}. Then, the CBCE algorithm with the standard OGD algorithm as its expert-algorithm and $\ell_1,\dots,\ell_T$ as the input loss functions achieves the following adaptive regret,
\begin{equation*}
\max\limits_{[q,s] \subseteq [T]} \left(\sum_{t=q}^{s}h_t(\x_t)-\min\limits_{\x\in\mathcal{X}}\sum_{t=q}^{s}h_t(\x)\right) \leq 15DG\sqrt{T}+8DG\sqrt{7\log T+5}\sqrt{T}.
\end{equation*}
\end{myThm}

\section{Proof of Lower Bound}
\label{sec:appendix-lower-bound}
We present the proof of the minimax lower bound of the universal dynamic regret for bandit convex optimization problems that established in Theorem~\ref{thm:lower-bound}.

\begin{proof}
For a given $\tau \in [0,2RT]$, we first construct a piecewise-stationary comparator sequence, whose path-length is constructed to be smaller than $\tau$. Then, we can split the whole time horizon into several pieces, where the comparator is fixed in each piece. Consequently, we are able to appeal to the established minimax lower bound of BCO in terms of static regret~\citep{conf/colt/DaniHK08,conf/colt/Shamir13} in each piece, and finally sum over all pieces to obtain the lower bound for the dynamic regret.

Follow the seminal work of~\citet{conf/colt/AbernethyBRT08} that provides the minimax lower bound for static regret, we adopt the notation of $R_T(\mathcal{X},\mathcal{F},\tau)$ to denote the minimax dynamic regret, defined as
\begin{equation}
  \label{eq:minimax-dynamic-regret}
  R_T(\mathcal{X},\mathcal{F},\tau) = \inf_{\x_1 \in \mathcal{X}} \sup_{f_1\in \mathcal{F}} \ldots \inf_{\x_T \in \mathcal{X}} \sup_{f_T\in \mathcal{F}} \left(\sum_{t=1}^{T} f_t(\x_t) - \min_{(\u_1,\ldots,\u_T) \in \mathcal{U}(\tau)}\sum_{t=1}^{T} f_t(\u_t)\right)
\end{equation}
where $\mathcal{F}$ denotes the set of convex functions that satisfies Assumption~\ref{assum:lipschitz-continuity}, and $\mathcal{U}(\tau) = \{(\u_1,\ldots,\u_T)  \ | \  \forall t\in[T], \u_t \in \mathcal{X}, \mbox{ and } P_T = \sum_{t=2}^{T} \norm{\u_{t-1} - \u_t}_2 \leq \tau\}$ is the set of feasible comparator sequences with path-length $P_T$ less than $\tau$.

We first consider the case of $\tau \leq 2R$. Then, we can utilize the established lower bound of the static regret for BCO problems~\citep{conf/colt/DaniHK08,conf/colt/Shamir13} as a natural lower bound of the dynamic regret,
\[
  R_T(\X,\F,\tau) \geq C_1 \cdot d RL\sqrt{T} = \frac{\sqrt{2}}{2} C_1 \cdot d L\sqrt{(R^2 + R^2) T} \geq C \cdot d L\sqrt{(R^2 + R\tau) T},
\]
where $C = \frac{\sqrt{2}}{2} C_1$, and $C_1$ is the constant appeared in the lower bound of static regret. The last inequality holds due to the condition $\tau \leq 2R$.

We next deal with the case of $\tau \geq 2R$. The idea is to construct a special comparator sequence in $\mathcal{U}(\tau)$, and split the whole time horizon into $K$ pieces such that the comparator sequence is fixed within each piece and only changes in the split point. Meanwhile, notice that the variation of the comparator sequence at each change point is $\tau/(K-1)$, at most $2R$. Combining these two observations, we have
\[
  \begin{split}
      R_T(\X,\F,\tau) & \geq K d RL\sqrt{\lceil T/K \rceil} \geq d RL\sqrt{KT} \geq d RL\sqrt{\left(\frac{\tau}{2R} + 1\right)T} \geq dL \sqrt{\frac{1}{2}(R^2 + R \tau) T},
  \end{split}
\]
which completes the proof.
\end{proof}

\section{Algorithm and Analysis of Adaptive Regret}
\label{sec:appendix-adaptive-regret}

In this section, we present algorithmic details and proofs of theoretical guarantees in Section~\ref{sec:adaptive}.

\subsection{Algorithm and Theoretical Guarantees}
Our proposed algorithm Minimizing Adaptive regret in Bandit Convex Optimization (\textsc{MABCO}) follows a similar framework to that of CBCE~\citep{AISTATS'17:coin-betting-adaptive}, which is a two-level structure, presented in Algorithm \ref{alg:master} (meta-algorithm) and Algorithm \ref{alg:expert} (expert-algorithm). However, we note that  a direct reduction of CBCE algorithm from the full-information setting to the bandit scenario by making use of the estimated gradients is prohibited, because the CBCE algorithm requires to query the loss function $O(\log t)$ times at each iteration $t$, which is not allowed in the bandit setup. 

To address this issue, we follow the same idea of the development of dynamic regret. Concretely, we introduce the surrogate loss function $\ell_t$ (defined in \eqref{surrogate-loss-1} and  \eqref{surrogate-loss-2} for different feedback models), whose function values as well as gradients can be computed by only using $f_t(\x_t)$ (or $f_t(\x^{(1)}_t)$ and $f_t(\x^{(2)}_t)$ for the two-point feedback model), without further queries of the loss function. We then deploy standard CBCE algorithm on surrogate loss functions series $\ell_1,\dots,\ell_{T}$ (Algorithm \ref{alg:master}). Based on the relationships between the surrogate loss $\ell_t$ and the original loss $f_t$, our proposed algorithm finally minimizes the expected adaptive regret on the original loss function sequence $f_1,\dots,f_T$.

\begin{algorithm}[t]
\caption{Minimizing Adaptive regret in Bandit Convex Optimization (MABCO)}
  \begin{algorithmic}[1]
  \label{alg:master}
\REQUIRE time horizon $T$, perturbation parameter $\delta$, shrinkage parameter $\alpha$
\STATE  Let $\mathcal{S}_1=\{E_1\}$, $q_i=1$, $\y_1=0$
\FOR{$t=1,...,T$}
    \FOR{$E_i\in$ $\mathcal{S}_{t}$}
    \IF{$q_i\not=t$}
        \STATE Pass the surrogate loss function $ \ell_{t}(\cdot)$ to expert $E_i$ (Algorithm \ref{alg:expert})
        \ENDIF
    \STATE Get the decision $\y_{i,t}$ of expert $E_i$
    \ENDFOR
    \STATE $\y_t=\sum_{E_i\in\mathcal{S}_t}p_{i,t}\y_{i,t}$
     \STATE {Select a unit vector $\mathbf{s}_t$ uniformly at random\\
     \{\textbf{Case 1.} One-Point Feedback Model\}}
  \STATE Submit $\x_{t} = \mathbf{y}_{t} + \delta \mathbf{s}_t$.
  \STATE {Observe $f_t(\x_t)$\\
    \{\textbf{Case 2.} Two-Point Feedback Model\}}
   \STATE Submit $\x^{(1)}_{t} = \mathbf{y}_{t} + \delta \mathbf{s}_t$ and $\x^{(2)}_{t} = \mathbf{y}_{t} - \delta \mathbf{s}_t$
   \STATE {Observe $f_t(\x^{(1)}_{t})$ and $f_t(\x^{(2)}_{t})$\\
   \{Adjust the expert set and update the weights\}}
\STATE Remove experts whose $e_i$ are less than $t$
    \FOR{$E_i\in$ $\mathcal{S}_{t}$}
    \STATE Compute $\widetilde{m}_{i,t}$ by \eqref{eq:compute-m-i}
    \ENDFOR
    \STATE Initialize $E_{\hat{n}}$, set $q_{\hat{n}}=t$ and compute ${e}_{\hat{n}}$
    \STATE $\hat{n}=|\mathcal{S}_t|+1$
    \STATE $\mathcal{S}_{t+1}= \mathcal{S}_t\cup\{E_{n}\}$
    \FOR{$E_i\in$ $\mathcal{S}_{t+1}$}
       \STATE Compute $w_{i,t+1}$ and $\hat{p}_{i,t+1}$ by \eqref{eq:compute-w-i} and \eqref{eq:compute-p-i}
    \ENDFOR
    \STATE $\textbf{p}_{t+1}=
 \begin{cases}
 \hat{\textbf{p}}_{t+1}/\|\hat{\textbf{p}}_{t+1}\|_1,& \|\hat{\textbf{p}}_{t+1}\|_1>0\\
 [\pi_{E_i}]_{E_i\in\mathcal{S}_t},& \text{otherwise}
 \end{cases}$
\ENDFOR
  \end{algorithmic}
\end{algorithm}
\begin{algorithm}[t]
\caption{Expert-algorithm}
  \begin{algorithmic}[1]
  \label{alg:expert}
\STATE Let $\widehat{G}=\max_{\y\in(1-\alpha)\mathcal{X},t\in[T]}\|\nabla\ell_{t}(\y)\|_2$.
\IF{$q_i=t$}
\STATE $\y_{i,t}=0$
\ELSE
\STATE{$\y_{i,t}= \mbox{Proj}_{(1-\alpha)\mathcal{X}}\left[\y_{i,t-1}-\frac{R}{\widehat{G}\sqrt{t-q_i}}\nabla \ell_{t-1}(\y_{i,t-1})\right]$}
\ENDIF
  \end{algorithmic}
\end{algorithm}

The detailed algorithm is described as follows. At iteration $t$, we maintain a set $\mathcal{S}_t$ of experts, each of which is an instantiation of the OGD algorithm (Algorithm \ref{alg:expert}), performing on surrogate loss function $\ell_t$. At the beginning of each iteration, we pass the surrogate loss function to experts and collect the predictions (line 3-8), then combine these predictions by their own weights (line 9). Next, we submit the perturbed decision and observe the feedback (line 11-12 for the one-point feedback model, and line 13-14 for the two-point feedback model ). Finally, we adjust the set of experts to get $\mathcal{S}_{t+1}$, and update the weights of experts in $\mathcal{S}_{t+1}$ according to their performance (line 15-25). Specifically, the (unnormalized) weight of expert $E_i$, i.e., $\hat{p}_{i,t+1}$, is computed by
\begin{equation}
\label{eq:compute-p-i}
\hat{p}_{i,t+1}=\pi_i\max\{w_{i,t+1},0\}
\end{equation}
where $\pi_i=1/\left(q_i^2(1+\lfloor \log q_i\rfloor)\right)$ is the prior of expert $E_i$,
\begin{equation}
\label{eq:compute-w-i}
w_{i,t+1}=\frac{\sum^t_{j=q_i}\tilde{m}_{j,t}}{t-q_i+1}\left(1+\sum_{j=q_i}^t\tilde{g}_{i,j}w_{i,j}\right)
\end{equation}
and
\begin{equation}
\label{eq:compute-m-i}
\tilde{m}_{i,t}=\ind_{w_{i,t}>0}(\ell_t(\y_t)-\ell_t(\y_{i,t}))+\ind_{w_{i,t}\leq0}\max\{\ell_{t}(\y_t)-\ell_{t}(\y_{i,t})\}.
\end{equation}

We refer to works of~\cite{AISTATS'17:coin-betting-adaptive} and~\cite{IJCAI:2018:Wang} for more details about the standard CBCE algorithm. Next, we provide an elaboration of the theoretical guarantees in Theorem~\ref{thm:bandit-adaptive-regret} as follows.
\begin{myThm}[one-point feedback model]
\label{thm:bandit-adaptive-regret-one-point}
Under Assumptions~\ref{assum:bounded-region},~\ref{assum:bounded-func-value}, and~\ref{assum:lipschitz-continuity}, define the surrogate loss function $\ell_t:(1-\alpha)\mathcal{X}\mapsto\mathbb{R}$ as
\begin{equation}
\label{surrogate-loss-1}
\ell_{t}(\y)=\frac{1}{2G^{one}R}\langle \tilde{g}_t,\y-\y_t \rangle+\frac{1}{2}
\end{equation}
where ${G}^{one}=dC/\delta$ and $\tilde{g}_t$ is the gradient estimator defined in~\eqref{eq:gradient-estimator-one-point}. Let Algorithm \ref{alg:master} be the meta-algorithm, which is fed with $\ell_1,\dots,\ell_T$ as loss functions, and Algorithm \ref{alg:expert} be the expert-algorithm. Set $\delta$ as in \eqref{optimal-delta} and $\alpha=\frac{\delta}{r}$. Then the expected adaptive regret satisfies
\begin{equation*}
\begin{split}
\mathbb{E}[{\AReg}_T]&\leq\sqrt{Cd\left(15R\sqrt{T}+8R\sqrt{7\log T+5}\sqrt{T}\right)\left(3LT+\frac{LR}{r}T\right)} = O\big(T^{\frac{3}{4}}(\log T)^{\frac{1}{4}}\big).
\end{split}
\end{equation*}
\end{myThm}
\begin{myThm}[two-point feedback model]
\label{thm:bandit-adaptive-regret-two-point}
Under Assumptions~\ref{assum:bounded-region},~\ref{assum:bounded-func-value}, and~\ref{assum:lipschitz-continuity}, define the surrogate loss function $\ell_t:(1-\alpha)\mathcal{X}\mapsto\mathbb{R}$ as
\begin{equation}
\label{surrogate-loss-2}
\ell_{t}(\y)=\frac{1}{2{G}^{two}R}\langle \tilde{g_t},\y-\y_t \rangle+\frac{1}{2}
\end{equation}
where ${G}^{two}=Ld$ and $\tilde{g}_t$ is the gradient estimator defined in~\eqref{eq:gradient-estimator-two-point}. Let Algorithm \ref{alg:master} be the meta-algorithm, which is fed with $\ell_1,\dots,\ell_T$ as loss functions, and Algorithm \ref{alg:expert} be the expert-algorithm. Set $\alpha=\delta/r$ and  $\delta=1/\sqrt{T}$. Then the expected adaptive regret satisfies
\begin{equation*}
\begin{split}
\mathbb{E}[{\AReg}_T]&\leq  Ld\left(15R\sqrt{T}+8R\sqrt{7\log T+5}\sqrt{T}\right)+3L\sqrt{T}+\frac{LR}{r}\sqrt{T}=O\big(T^{\frac{1}{2}}(\log T)^{\frac{1}{2}}\big).
\end{split}
\end{equation*}
\end{myThm}

\subsection{Proof of Theorem \ref{thm:bandit-adaptive-regret-one-point}}
\label{proof-of-theorem1}
\begin{proof}
For any time interval $I=[q,s] \subseteq [T]$, we have
\begin{align}
  {} &\mathbb{E}\left[\sum_{t=q}^{s}f_t(\x_t)\right]-\min\limits_{\x\in\mathcal{X}}\sum_{t=q}^{s}f_t(\x) \nonumber\\
= {} & \underbrace{\mathbb{E}\left[\sum_{t=q}^s\hat{f}_t(\y_t)\right]-\min\limits_{\y\in(1-\alpha)\mathcal{X}}\sum_{t=q}^{s} \hat{f}_t(\y)}_{\mathtt{term (a)}}+\underbrace{\mathbb{E}\left[\sum_{t=q}^sf_t(\x_t)-\hat{f}_t(\y_t)\right]}_{\mathtt{term (b)}} + \underbrace{\min\limits_{\x\in(1-\alpha)\mathcal{X}}\sum_{t=q}^s\hat{f}_t(\x)-\min\limits_{\x\in\mathcal{X}}f_t(\x)}_{\mathtt{term (c)}}\nonumber\\
\leq {} &\ {\mathtt{term (a)}}+3L\delta T+\frac{LR}{r}\delta T \label{eq:proof-of-theorem-adaptive-regret-one}
\end{align}
where \eqref{eq:proof-of-theorem-adaptive-regret-one} follows from the analysis in dynamic regret (see \eqref{eq:diff-1} and \eqref{eq:diff-2}). Note that since $\y_t$ is the weighted combination of $\y_{i,t}$, it still satisfies $\y_t\in(1-\alpha)\mathcal{X}$.

Now, it remains to bound term (a). Define the function $h_t:(1-\alpha)\mathcal{X}\mapsto\mathbb{R}$ by $h_t(\y)=\hat{f}_t(\y)+\inner{\y}{\xi_t}$, where $\xi_t = \gt_t - \nabla \fh_t(\y_t)$ with $\gt_t = \frac{d}{\delta}f_t(\y_t+\delta \s_t)\cdot \s_t$. By the analysis of dynamic regret (see~\eqref{eq:expectation-oblivious}), we know that $\E[h_t(\y)]=\mathbb{E}[\hat{f}_t(\y)]$ for any fixed $\y \in (1-\alpha)\X$. Besides, since $\nabla h_t(\y_t)= \nabla \hat{f}_t(\y_t) + \xi_t = \tilde{g}_t$, the following holds for any $\y \in (1-\alpha)\X$,
\[
	h_t(\y_t)-h_t(\y)\leq \nabla h_t(\y_t)^\T (\y_t - \y) \overset{\eqref{surrogate-loss-1}}{=} -2{G}^{one}R\ell_{t}(\y)+{G}^{one}R.
\]

Note that since $\ell_t(\y_t)=\frac{1}{2}$, we know that for any $\y\in{(1-\alpha)\mathcal{X}}$,
\begin{equation}
\label{relation-surrogate-1}
\mathbb{E}\left[\hat{f}_t(\y_t)-\hat{f}_t(\y)\right]=\mathbb{E}\left[h_t(\y_t)-h_t(\y)\right]\leq 2{G}^{one}R\cdot\mathbb{E}\left[(\ell_t(\y_t)-\ell_t(\y))\right].
\end{equation}

On the other hand, Algorithm \ref{alg:master} is essentially a standard CBCE algorithm deploying on a full-information online learning problem where the loss function sequence is $\ell_1,\dots,\ell_T$. Hence, Theorem~\ref{thm:saregret} implies
\begin{equation*}
\max\limits_{[q,s] \subseteq [T]} \left(\sum_{t=q}^{s}\ell_t(\y_t)-\min\limits_{\y\in(1-\alpha)\mathcal{X}}\sum_{t=q}^{s}\ell_t(\y)\right) \leq 15R\hat{G}\sqrt{T}+8\sqrt{7\log T+5}\sqrt{T}
\end{equation*}
where $\hat{G}=\sup_{\y\in(1-\alpha)\mathcal{X}, t\in [T]}\norm{\ell_{t}(\y)}_2\leq \frac{1}{2R}$. This in conjunction with~\eqref{relation-surrogate-1} yields
\begin{equation}
\label{eq:terma}
\max\limits_{[q,s] \subseteq [T]} \left(\mathbb{E}\left[\sum_{t=q}^s\hat{f}_t(\y_t)\right]-\min\limits_{\y\in(1-\alpha)\mathcal{X}}\sum_{t=q}^{s} \hat{f}_t(\y)\right)\leq 15{G}^{one}R\sqrt{T}+8{G}^{one}R\sqrt{7\log T+5}\sqrt{T}.
\end{equation}

Plugging \eqref{eq:terma} into \eqref{eq:proof-of-theorem-adaptive-regret-one}, we get
\begin{align}
	 {} & \mathbb{E}\left[\sum_{t=q}^{s}f_t(\x_t)\right]-\min\limits_{\x\in\mathcal{X}}\sum_{t=q}^{s}f_t(\x)\nonumber \\
\leq {} & 15{G}^{one}R\sqrt{T}+8{G}^{one}R\sqrt{7\log T+5}\sqrt{T}+3L\delta T+\frac{LR}{r} \delta T\nonumber\\
\leq {} & \frac{Cd}{\delta}\left(15R\sqrt{T}+8R\sqrt{7\log T+5}\sqrt{T}\right)+\delta\left(3LT+\frac{LR}{r}T\right)\nonumber\\
=	 {} & \sqrt{Cd\left(15R\sqrt{T}+8R\sqrt{7\log T+5}\sqrt{T}\right)\left(3LT+\frac{LR}{r}T\right)} \label{eq:proof-adaptive-regret-th1-fin}\\
=	 {} & O\big(T^{\frac{3}{4}}(\log T)^{\frac{1}{4}}\big)\nonumber
\end{align}
where \eqref{eq:proof-adaptive-regret-th1-fin} is derived by optimally configuring
\begin{equation}
\label{optimal-delta}
\delta=\sqrt{\frac{Cd(15R\sqrt{T}+8R(\sqrt{7\log T+5}\sqrt{T}))}{3LT+LRT/r}}
\end{equation}
which finishes the proof.
\end{proof}

\subsection{Proof of Theorem~\ref{thm:bandit-adaptive-regret-two-point}}
\begin{proof}
The proof is similar to that in Section~\ref{proof-of-theorem1}. Define the function $h_t:(1-\alpha)\mathcal{X}\mapsto\mathbb{R}$ by $h_t(\y)=\hat{f}_t(\y)+\y^\T\xi_t$, where $\xi_t = \gt_t - \nabla \fh_t(\y_t)$ with $\gt_t = \frac{d}{2\delta} \left(f_t(\y_t + \delta\s_t) - f_t(\y_t - \delta\s_t)\right)\cdot \s_t$. Similarly, $\E[h_t(\y)]=\mathbb{E}[\hat{f}_t(\y)]$ holds for any fixed $\y \in (1-\alpha)\X$. Besides, since $\nabla h_t(\y_t)= \nabla \hat{f}_t(\y_t) + \xi_t = \tilde{g}_t$, we have $\forall \y \in (1-\alpha)\X$,
\[
	h_t(\y_t)-h_t(\y)\leq -2{G}^{two}R\ell_{t}(\y)+{G}^{two}R.
\]
Note that since $\ell_t(\y_t)=\frac{1}{2}$, we have $\forall\y\in{(1-\alpha)\mathcal{X}},$
\begin{equation}
\label{relation-surrogate-2}
\mathbb{E}\left[\hat{f}_t(\y_t)-\hat{f}_t(\y)\right]=\mathbb{E}\left[h_t(\y_t)-h_t(\y)\right]\leq 2{G}^{two}R\mathbb{E}\left[(\ell_t(\y_t)-\ell_t(\y))\right].
\end{equation}
Hence, by deploying the standard CBCE algorithm on the loss function series $\ell_{1},\ldots,\ell_T$ (Algorithm~\ref{alg:master}), and based on Theorem~\ref{thm:saregret}, we have
\begin{equation}
\begin{split}
\label{eq:saregret1}
\max\limits_{[q,s] \subseteq [T]} \left(\sum_{t=q}^{s}\ell_t(\y_t)-\min\limits_{\y\in(1-\alpha)\mathcal{X}}\sum_{t=q}^{s}\ell_t(\y)\right)&\leq 15R\hat{G}\sqrt{T}+8\sqrt{7\log T+5}\sqrt{T}
\end{split}
\end{equation}
where $\hat{G}=\max_{\y\in(1-\alpha)\mathcal{X}, t\in [T]}\|\nabla\ell_{t}(\y)\|_2\leq \frac{1}{2R}$. Thus, we have
\begin{align}
 	 {} &\mathbb{E}\left[\sum_{t=q}^{s}f_t(\x_t)\right]-\min\limits_{\x\in\mathcal{X}}\sum_{t=q}^{s}f_t(\x)\nonumber \\
\leq {} &15{G}^{two}R\sqrt{T}+8{G}^{two}R\sqrt{7\log T+5}\sqrt{T}+3L\delta T+\frac{LR}{r} \delta T\tag*{(by setting $\delta = 1/\sqrt{T}$)}\\
\leq {} & Ld\left(15R\sqrt{T}+8R\sqrt{7\log T+5}\sqrt{T}\right)+\delta\left(3LT+\frac{LR}{r}T\right)\nonumber\\
= 	 {} & Ld\left(15R\sqrt{T}+8R\sqrt{7\log T+5}\sqrt{T}\right)+3L\sqrt{T}+\frac{LR}{r}\sqrt{T}\nonumber\\
= 	 {} & O(\sqrt{T\log T})\nonumber
\end{align}
Therefore, we complete the proof.
\end{proof}
\end{document}